\documentclass[final,12pt]{arxiv}

\usepackage[utf8]{inputenc} % allow utf-8 input
\usepackage[T1]{fontenc}    % use 8-bit T1 fonts
\usepackage{hyperref}       % hyperlinks
\usepackage{url}            % simple URL typesetting
\usepackage{booktabs}       % professional-quality tables
\usepackage{amsfonts}       % blackboard math symbols
\usepackage{nicefrac}       % compact symbols for 1/2, etc.
\usepackage{microtype}      % microtypography
\usepackage{xcolor}         % colors

\usepackage{dsfont}
\usepackage{enumitem}
\usepackage[toc,page,header]{appendix}
\usepackage{lipsum,wrapfig}

\usepackage[english]{babel}
\usepackage{amsmath, amssymb}
\usepackage{graphicx}
\usepackage{latexsym}
\usepackage{graphicx}
\usepackage{cases}
\usepackage[mathscr]{euscript}
\usepackage{xcolor}
\usepackage{colortbl}
\usepackage{thmtools}
\usepackage{thm-restate}

\usepackage{mathtools}
\usepackage{subcaption}

\usepackage{multirow}
\usepackage{wrapfig}

\usepackage{pifont}% http://ctan.org/pkg/pifont

\let\vec\undefined
\usepackage{MnSymbol}
\DeclareMathAlphabet\mathbb{U}{msb}{m}{n}
\usepackage{xpatch}

\def\Rset{\mathbb{R}}

\DeclareMathOperator*{\E}{\mathbb E}
\DeclareMathOperator*{\argmax}{argmax}
\DeclareMathOperator*{\argmin}{argmin}

\DeclarePairedDelimiter{\bracket}{[}{]}
\DeclarePairedDelimiter{\curl}{\{}{\}}

\DeclarePairedDelimiter{\paren}{(}{)}

\newcommand{\sA}{{\mathscr A}}

\newcommand{\sC}{{\mathscr C}}
\newcommand{\sD}{{\mathscr D}}
\newcommand{\sE}{{\mathscr E}}
\newcommand{\sF}{{\mathscr F}}

\newcommand{\sH}{{\mathscr H}}

\newcommand{\sK}{{\mathscr K}}

\newcommand{\sM}{{\mathscr M}}

\newcommand{\sR}{{\mathscr R}}
\newcommand{\sS}{{\mathscr S}}

\newcommand{\sX}{{\mathscr X}}
\newcommand{\sY}{{\mathscr Y}}

\newcommand{\ov}{\overline}
\newcommand{\wt}{\widetilde}
\newcommand{\e}{\epsilon}

\newcommand{\ignore}[1]{}

\def\Nset{\mathbb{N}}
\newcommand{\hh}{{\sf h}}
\newcommand{\rr}{{\sf r}}
\newcommand{\pp}{{\sf p}}

\hypersetup{
  breaklinks   = true, %splits links across lines
  colorlinks   = true, %Colours links instead of ugly boxes
  urlcolor     = blue, %Colour for external hyperlinks
  linkcolor    = blue, %Colour of internal links
  citecolor   = blue %Colour of citations
}

\usepackage[toc, page, header]{appendix}
\usepackage{times}
\title[Top-$k$ Classification and Cardinality-Aware Prediction]
      {Top-\texorpdfstring{$k$}{k} Classification and Cardinality-Aware Prediction}

\arxivauthor{%
 \Name{Anqi Mao} \Email{aqmao@cims.nyu.edu}\\
 \addr Courant Institute of Mathematical Sciences, New York%
 \AND
 \Name{Mehryar Mohri} \Email{mohri@google.com}\\
 \addr Google Research and Courant Institute of Mathematical Sciences, New York%
 \AND
 \Name{Yutao Zhong} \Email{yutao@cims.nyu.edu}\\
 \addr Courant Institute of Mathematical Sciences, New York%
}

\begin{document}

\maketitle

\begin{abstract}

We present a detailed study of top-$k$ classification, the task of
predicting the $k$ most probable classes for an input, extending
beyond single-class prediction. We demonstrate that several prevalent
surrogate loss functions in multi-class classification, such as
comp-sum and constrained losses, are supported by $\sH$-consistency
bounds with respect to the top-$k$ loss. These bounds guarantee
consistency in relation to the hypothesis set $\sH$, providing
stronger guarantees than Bayes-consistency due to their non-asymptotic
and hypothesis-set specific nature. To address the trade-off between
accuracy and cardinality $k$, we further introduce cardinality-aware
loss functions through instance-dependent cost-sensitive learning. For
these functions, we derive cost-sensitive comp-sum and constrained
surrogate losses, establishing their $\sH$-consistency bounds and
Bayes-consistency.
Minimizing these losses leads to new cardinality-aware algorithms for
top-$k$ classification. We report the results of extensive experiments
on CIFAR-100, ImageNet, CIFAR-10, and SVHN datasets demonstrating
the effectiveness and benefit of these algorithms.

\ignore{
We present a detailed study of top-$k$ classification, which consists
of predicting the $k$ most likely classes for a given input, as
opposed to solely predicting the single most likely class.  We show
that remarkably, the common families of surrogate losses used in
standard multi-class classification, including the comp-sum losses and
the constrained losses benefit from $\sH$-consistency bounds with
respect to the top-$k$ loss. These are recently proposed consistency
guarantees that are more relevant to learning than Bayes-consistency,
as they are non-asymptotic and specific to the hypothesis set $\sH$
adopted. As a by-product, these guarantees also imply the
Bayes-consistency of these surrogate losses. We then introduce target
cardinality-aware loss functions for determining the optimal
cardinality $k$ of top-$k$ classification via instance-dependent
cost-sensitive learning, for which we propose two novel families of
surrogate losses: cost-sensitive comp-sum losses and cost-sensitive
constrained losses. We prove $\sH$-consistency bounds and thus
Bayes-consistency for these cost-sensitive surrogate losses with
respect to the cardinality-aware losss functions. Our extensive
experiments on CIFAR-100, ImageNet, CIFAR-10 and SVHN datasets
illustrate the effectiveness of cardinality-aware algorithms by
minimizing these loss functions.
}
\end{abstract}

% \begin{keywords}%

% \end{keywords}

\section{Introduction}
\label{sec:intro}

Top-$k$ classification consists of predicting the $k$ most likely
classes for a given input, as opposed to solely predicting the single
most likely class. Several compelling reasons support the adoption of
top-$k$ classification. First, it enhances accuracy by allowing the
model to consider the top $k$ predictions, accommodating uncertainty
and providing a more comprehensive prediction. This proves
particularly valuable in scenarios where multiple correct answers
exist, such as image tagging, where a top-$k$ classifier can identify
all relevant objects in an image. Furthermore, top-$k$ classification
finds application in ranking and recommendation tasks, like suggesting
the top $k$ most relevant products in e-commerce based on user
queries. The confidence scores associated with the top $k$ predictions
also serve as a means to estimate the model's uncertainty, a crucial
aspect in applications requiring insight into the model's confidence
level.

Ensembling can also benefit from top-$k$ predictions as they can be
combined from multiple models, contributing to improved overall
performance by introducing a more robust and diverse set of
predictions. In addition, top-$k$ predictions can serve as input for
downstream tasks like natural language generation or dialogue systems,
enhancing the performance of these tasks by providing a broader range
of potential candidates. Finally, the interpretability of the model's
decision-making process is enhanced by examining the top $k$ predicted
classes, allowing users to gain insights into the rationale behind the
model's predictions.

However, the top-$k$ loss function is non-continuous and
non-differentiable, and its direct optimization is
intractable. Therefore, top-$k$ classification algorithms typically
resort to a surrogate loss \citep{lapin2015top,lapin2016loss,
  berrada2018smooth,reddi2019stochastic,
  yang2020consistency,thilagar2022consistent}.  This raises critical
questions: Which surrogate loss functions admit theoretical guarantees
and efficient minimization properties?  Can we design accurate top-$k$
classification algorithms?

Unlike standard classification, this problem has been relatively
unexplored. A crucial property in this context is
\emph{Bayes-consistency}, which has been extensively studied in binary
and multi-class classification
\citep{Zhang2003,bartlett2006convexity,zhang2004statistical,
  bartlett2008classification}. While Bayes-consistency has been
explored for various top-$k$ surrogate losses
\citep{lapin2015top,lapin2016loss,lapin2018analysis,yang2020consistency,thilagar2022consistent},
some face limitations.  Non-convex "hinge-like" surrogates
\citep{yang2020consistency}, inspired by ranking
\citep{usunier2009ranking}, and polyhedral surrogates
\citep{thilagar2022consistent} cannot lead to effective algorithms as
they cannot be efficiently computed and optimized.  Negative results
indicate that several convex "hinge-like" surrogates
\citep{lapin2015top,lapin2016loss,lapin2018analysis} fail to achieve
Bayes-consistency \citep{yang2020consistency}. Can we shed more light
on these results?

On the positive side, it has been shown that the logistic loss (or
cross-entropy loss used with the softmax activation) is a
Bayes-consistent loss for top-$k$ classification
\citep{lapin2015top,yang2020consistency}. This prompts further
inquiries: Which other smooth loss functions admit this property? More
importantly, can we establish non-asymptotic and hypothesis
set-specific guarantees for these surrogates, quantifying their
effectiveness?  Beyond top-$k$ classification, it is important to
consider the trade-off between accuracy and the cardinality $k$. This
leads us to introduce and study cardinality-aware top-$k$
classification algorithms, which aim to achieve a high accuracy while
maintaining a small average cardinality.

This paper presents a detailed study of top-$k$ classification. We
first show that, remarkably, several widely used families of surrogate
losses used in standard multi-class classification admit
\emph{$\sH$-consistency bounds} \citep{awasthi2022h,awasthi2022multi,
  mao2023cross,MaoMohriZhong2023characterization} with respect to the
top-$k$ loss.  These are strong consistency guarantees that are
non-asymptotic and specific to the hypothesis set $\sH$ adopted, which
further imply Bayes-consistency.  In Section~\ref{sec:comp}, we
demonstrate this property for the broad family of \emph{comp-sum
losses} \citep{mao2023cross}, which includes the logistic loss, the
sum-exponential loss, the mean absolute error loss, and the
generalized cross-entropy loss. Further, in Section~\ref{sec:cstnd},
we prove it for \emph{constrained losses}, originally introduced for
multi-class SVM \citep{lee2004multicategory}, including the
constrained exponential loss, constrained hinge loss and squared hinge
loss, and the $\rho$-margin loss.
These guarantees provide a strong foundation for principled algorithms
in top-$k$ classification, leveraging the minimization of these
surrogate loss functions. Many of these loss functions are known for
their smooth properties and favorable optimization solutions.

In Section~\ref{sec:cardinality}, we further investigate
cardinality-aware top-$k$ classification, aiming to return an accurate
top-$k$ list with the lowest average cardinality $k$ for
each input instance.  We introduce a target loss function tailored to
this problem through instance-dependent cost-sensitive learning
(Section~\ref{sec:cost-learning}).
Subsequently, we present two novel surrogate loss families
for optimizing this target loss: cost-sensitive comp-sum losses
(Section~\ref{sec:cost-comp}) and cost-sensitive constrained losses
(Section~\ref{sec:cost-cstnd}). These loss functions are
obtained by augmenting their standard counterparts
with instance-dependent cost terms.
We establish $\sH$-consistency bounds and thus Bayes-consistency for
these cost-sensitive surrogate losses with respect to the
cardinality-aware target loss. Minimizing these losses leads to new
cardinality-aware algorithms for top-$k$ classification.
Section~\ref{sec:experiments} presents experimental results on
CIFAR-100, ImageNet, CIFAR-10, and SVHN datasets, demonstrating the
effectiveness of these algorithms.

\section{Preliminaries}
\label{sec:pre}

We consider the learning task of top-$k$ classification with $n \geq
2$ classes, that is seeking to ensure that the correct class label
for a given input sample is among the top $k$\textbf{} predicted classes.
We denote by $\sX$ the input space and $\sY = [n] \colon = \curl*{1,
  \ldots, n}$ the label space.  We denote by $\sD$ a distribution over
$\sX \times \sY$ and write $p(x, y) = \sD\paren*{Y = y \mid X = x}$ to
denote the conditional probability of $Y = y$ given $X = x$. We also
write $p(x) = \paren*{p(x, 1), \ldots, p(x, n)}$ to denote the
corresponding conditional probability vector.

We denote by $\ell \colon \sH_{\rm{all}} \times \sX \times \sY \to
\Rset$ a loss function defined for the family of all measurable
functions $\sH_{\rm{all}}$.  Given a hypothesis set $\sH \subseteq
\sH_{\rm{all}}$, the conditional error of a hypothesis $h$ and the
best-in-class conditional error are defined as follows:
\begin{align*}
  \sC_{\ell}(h, x)
  & = \E_{y \mid x}\bracket*{\ell(h, x, y)}
  = \sum_{y \in \sY} p(x, y) \ell(h, x, y)\\
  \sC^*_{\ell}(\sH, x)
  & = \inf_{h \in \sH} \sC_{\ell}(h, x)
  = \inf_{h \in \sH} \sum_{y \in \sY} p(x, y) \ell(h, x, y).
\end{align*}
Accordingly, the generalization error of a hypothesis $h$ and the
best-in-class generalization error are defined by:
\begin{align*}
  \sE_{\ell}(h) &= \E_{(x, y) \sim \sD} \bracket*{\ell(h, x, y)}
  = \mathbb{E}_x \bracket*{\sC_{\ell}(h, x)}\\
  \sE^*_{\ell}(\sH) &= \inf_{h \in \sH} \sE_{\ell}(h)
  = \inf_{h \in \sH} \mathbb{E}_x \bracket*{\sC_{\ell}(h, x)}.   
\end{align*}
Given a score vector $\paren*{h(x, 1), \ldots, h(x, n)}$ generated by
hypothesis $h$, we sort its components in decreasing order and write
$\hh_k(x)$ to denote the $k$th label, that is $h(x, \hh_1(x)) \geq
h(x, \hh_2(x)) \geq \ldots \geq h(x, \hh_{n - 1}(x)) \geq h(x,
\hh_n(x))$. Similarly, for a given conditional probability vector
$p(x) = \paren*{p(x, 1), \ldots, p(x, n)}$, we write $\pp_k(x)$ to
denote the $k$th element in decreasing order, that is $p(x,\pp_1(x))
\geq p(x,\pp_2(x)) \geq \ldots \geq p(x, \pp_n(x))$. In the event of a
tie for the $k$-th highest score or conditional probability, the label
$\hh_k(x)$ or $\pp_k(x)$ is selected based on the highest index when
considering the natural order of labels.

The target generalization error for top-$k$ classification is given by
the top-$k$ loss, which is denoted by $\ell_{k}$ and defined, for any
hypothesis $h$ and $(x, y) \in \sX \times \sY$ by
\begin{equation*}
\ell_{k}(h, x, y) = 1_{y \notin \curl*{\hh_1(x), \ldots, \hh_k(x)}}.   
\end{equation*}
Thus, the loss takes value one when the correct label $y$ is not
included in the top-$k$ predictions made by the hypothesis $h$, zero
otherwise. In the special case where $k = 1$, this is precisely the
familiar zero-one classification loss. As with the zero-one loss,
optimizing the top-$k$ loss is NP-hard for common hypothesis
sets. Therefore, an alternative surrogate loss is typically used
to design learning algorithms.

A crucial property of these surrogate losses is
\emph{Bayes-consistency}. This requires that, asymptotically, nearly
minimizing a surrogate loss over the family of all measurable
functions leads to the near minimization of the top-$k$ loss over the
same family \citep{steinwart2007compare}.
\begin{definition}
A surrogate loss $\ell$ is said to be \emph{Bayes-consistent with
respect to the top-$k$ loss $\ell_k$} if, for all given sequences of
hypotheses $\curl*{h_n}_{n \in \Nset} \subset \sH_{\rm{all}}$ and any
distribution, $\lim_{n \to \plus \infty} \sE_{\ell}\paren*{h_n} -
\sE^*_{\ell}\paren*{\sH_{\rm{all}}} = 0$ implies $\lim_{n \to \plus
  \infty} \sE_{\ell_{k}}\paren*{h_n} -
\sE^*_{\ell_{k}}\paren*{\sH_{\rm{all}}} = 0$.
\end{definition}
Bayes-consistency is an asymptotic guarantee and applies only to the
family of all measurable functions. Recently,
\citet*{awasthi2022h,awasthi2022multi} proposed a stronger consistency
guarantee, referred to as \emph{$\sH$-consistency bounds}. These are
upper bounds on the target estimation error in terms of the surrogate
estimation error that are non-asymptotic and hypothesis set-specific
guarantees.
\begin{definition}
Given a hypothesis set $\sH$, a surrogate loss $\ell$ is said to admit
an $\sH$-consistency bound with respect to the top-$k$ loss $\ell_k$
if, for some non-decreasing function $f$, the following inequality
holds for all $h \in \sH$ and for any distribution:
\begin{equation*}
  f \paren*{\sE_{\ell_{k}}\paren*{h} - \sE^*_{\ell_{k}}\paren*{\sH}}
  \leq \sE_{\ell}\paren*{h} - \sE^*_{\ell} \paren*{\sH}.
\end{equation*}
\end{definition}
We refer to $\sE_{\ell_{k}}\paren*{h} - \sE^*_{\ell_{k}}\paren*{\sH}$ as the target estimation error and $\sE_{\ell}\paren*{h} - \sE^*_{\ell} \paren*{\sH}$ as the surrogate estimation error. These bounds imply Bayes-consistency when $\sH = \sH_{\rm{all}}$, by
taking the limit on both sides.

We will study $\sH$-consistency bounds for common surrogate losses in
the multi-class classification, with respect to the top-$k$ loss
$\ell_k$. A key quantity appearing in $\sH$-consistency bounds is the
\emph{minimizability gap}, which measures the difference between the
best-in-class generalization error and the expectation of the
best-in-class conditional error, defined for a given hypothesis set
$\sH$ and a loss function $\ell$ by:
\begin{equation*}
\sM_{\ell}(\sH) = \sE^*_{\ell}(\sH) - \mathbb{E}_x \bracket*{\sC^*_{\ell}(\sH, x)}.
\end{equation*}
As shown by \citet{mao2023cross}, the minimizability gap is
non-negative and is upper bounded by the approximation error
$\sA_{\ell}(\sH) = \sE^*_{\ell}(\sH) - \sE^*_{\ell}(\sH_{\rm{all}})$:
$0 \leq \sM_{\ell}(\sH) \leq \sA_{\ell}(\sH)$. When $\sH =
\sH_{\rm{all}}$ or more generally $\sA_{\ell_{\log}}(\sH) = 0$, the
minimizability gap vanishes. However, in general, it is non-zero and
provides a finer measure than the approximation error. Thus,
$\sH$-consistency bounds provide a stronger guarantee than the excess
error bounds.

We will specifically study the surrogate loss families of
\emph{comp-sum losses} and \emph{constrained losses} in multi-class
classification, which have been shown in the past to benefit from
$\sH$-consistency bounds with respect to the zero-one classification
loss, that is $\ell_k$ with $k = 1$
\citep{awasthi2022multi,mao2023cross} (see also
\citep{MaoMohriZhong2023ranking,MaoMohriZhong2023rankingabs,MaoMohriZhong2023structured,MaoMohriMohriZhong2023twostage,zheng2023revisiting,MaoMohriZhong2024deferral,MaoMohriZhong2024score,MaoMohriZhong2024predictor,MohriAndorChoiCollinsMaoZhong2024learning}). We will significantly extend these
results to top-$k$ classification and prove $\sH$-consistency bounds
for these loss functions with respect to $\ell_k$ for any $1 \leq k
\leq n$.

Note that another commonly used family of surrogate losses in
multi-class classification is the \emph{max losses}, which are defined
through a convex function, such as the hinge loss function applied to
the margin \citep{crammer2001algorithmic,awasthi2022multi}. However,
as shown in \citep{awasthi2022multi}, no non-trivial $\sH$-consistency
guarantee holds for max losses with respect to $\ell_k$, even when $k
= 1$.

We first characterize the best-in class conditional error and the
conditional regret of top-$k$ loss, which will be used in the analysis
of $\sH$-consistency bounds.  We denote by $S^{[k]} = \curl*{X \subset
  S \mid |X| = k}$ the set of all $k$-subsets of a set $S$. We will
study any hypothesis set that is regular.
\begin{definition}
  Let $A(n, k)$ be the set of ordered $k$-tuples with distinct
  elements in $[n]$.  We say that a hypothesis set $\sH$ is
  \emph{regular for top-$k$ classification}, if the top-$k$
  predictions generated by the hypothesis set cover all possible
  outcomes:
\begin{equation*}
  \forall x \in \sX,\,
  \curl*{(\hh_1(x), \dots, \hh_k(x)) \colon h \in \sH} = A(n, k).
\end{equation*}
\end{definition}
Common hypothesis sets such as that of linear models or neural
networks, or the family of all measurable functions, are all regular
for top-$k$ classification.

\begin{restatable}{lemma}{RegretTarget}
\label{lemma:regret-target}
Assume that $\sH$ is regular. Then, for any $h \in \sH$ and $x \in
\sX$, the best-in class conditional error and the conditional regret
of the top-$k$ loss can be expressed as follows:
\begin{align*}
  \sC^*_{\ell_k}(\sH, x)
  & = 1 - \sum_{i = 1}^k p(x, \pp_i(x))\\
\Delta \sC_{\ell_k, \sH}(h, x)
& = \sum_{i = 1}^k \paren*{p(x, \pp_i(x)) - p(x, \hh_i(x))}.
\end{align*}
\end{restatable}
The proof is included in Appendix~\ref{app:regret-target}. Note that,
for $k = 1$, the result coincides with the known identities for
standard multi-class classification with regular hypothesis sets
\citep[Lemma~3]{awasthi2022multi}.

As with \citep{awasthi2022multi,mao2023cross}, in the following
sections, we will consider hypothesis sets that are symmetric and
complete. This includes the class of linear models and neural networks
typically used in practice, as well as the family of all measurable
functions.
We say that a hypothesis set $\sH$ is \emph{symmetric} if it is
independent of the ordering of labels. That is, for all $y \in \sY$,
the scoring function $x \mapsto h(x, y)$ belongs to some real-valued
family of functions $\sF$. We say that a hypothesis set is
\emph{complete} if, for all $(x, y) \in \sX \times \sY$, the set of
scores $h(x, y)$ can span over the real numbers, that is, $\curl*{h(x,
  y) \colon h \in \sH} = \Rset$. Note that any symmetric and complete
hypothesis set is regular for top-$k$ classification.

Next, we analyze the broad family of \emph{comp-sum losses}, which
includes the commonly used logistic loss (or cross-entropy loss used
with the softmax activation) as a special case.

\section{$\sH$-Consistency Bounds for Comp-Sum Losses}
\label{sec:comp}

Comp-sum losses are defined as the composition of a function $\Phi$
with the sum exponential losses, as shown in \citep{mao2023cross}. For
any $h \in \sH$ and $ (x, y) \in \sX \times \sY$, they are expressed
as
\begin{align*}
\ell^{\rm{comp}}(h, x, y)
= \Phi\paren*{ \sum_{y' \neq y}e^{ h(x, y') - h(x, y) } },
\end{align*}
where $\Phi \colon \Rset_{+} \to \Rset_{+} $ is non-decreasing. When
$\Phi$ is chosen as the function $t \mapsto \log(1 + t)$, $t \mapsto
t$, $t \mapsto 1 - \frac{1}{1 + t}$ and $t \mapsto \frac{1}{\alpha}
\paren*{1 - \paren*{\frac{1}{1 + t}}^{\alpha} }$, $\alpha \in (0, 1)$,
$\ell^{\rm{comp}}(h, x, y)$ coincides with the (multinomial) logistic
loss $\ell_{\rm{log}}$
\citep{Verhulst1838,Verhulst1845,Berkson1944,Berkson1951}, the
sum-exponential loss $\ell^{\rm{comp}}_{\rm{exp}}$
\citep{weston1998multi,awasthi2022multi}, the mean absolute error loss
$\ell_{\rm{mae}}$ \citep{ghosh2017robust}, and the generalized cross
entropy loss $\ell_{\rm{gce}}$ \citep{zhang2018generalized},
respectively. We we will specifically study these loss functions and
show that they benefit from $\sH$-consistency bounds with respect to
the top-$k$ loss. 

\subsection{Logistic loss}

We first show that the most commonly used logistic loss, defined as
$\ell_{\rm{log}}(h, x, y) = \log \paren*{\sum_{y' \in \sY}e^{ h(x, y')
    - h(x, y) }}$, admits $\sH$-consistency bounds with respect to
$\ell_{k}$.
\begin{restatable}{theorem}{BoundLog}
\label{thm:bound-log}
Assume that $\sH$ is symmetric and complete. Then, for any $1 \leq k
\leq n$, the following $\sH$-consistency bound holds for the logistic
loss:
\ifdim\columnwidth=\textwidth
{
\begin{align*}
\sE_{\ell_k}(h) - \sE^*_{\ell_k}(\sH) + \sM_{\ell_k}(\sH)
 \leq k \psi^{-1} \paren*{ \sE_{\ell_{\log}}(h)
    - \sE^*_{\ell_{\log}}(\sH) + \sM_{\ell_{\log}}(\sH) },
\end{align*}
}
\else{
\begin{align*}
& \sE_{\ell_k}(h) - \sE^*_{\ell_k}(\sH) + \sM_{\ell_k}(\sH)\\
  & \quad \leq k \psi^{-1} \paren*{ \sE_{\ell_{\log}}(h)
    - \sE^*_{\ell_{\log}}(\sH) + \sM_{\ell_{\log}}(\sH) },
\end{align*}
}
\fi
where $\psi(t) = \frac{1 - t}{2}\log(1 - t) + \frac{1 + t}{2}\log(1+
t)$, $t \in [0,1]$.  In the special case where $\sA_{\ell_{\log}}(\sH)
= 0$, for any $1 \leq k \leq n$, the following upper bound holds:
\begin{align*}
  \sE_{\ell_k}(h) - \sE^*_{\ell_k}(\sH)
  \leq k \psi^{-1} \paren*{ \sE_{\ell_{\log}}(h) - \sE^*_{\ell_{\log}}(\sH) }.
\end{align*}
\end{restatable}
The proof is included in Appendix~\ref{app:bound-log}. The second part
follows from the fact that when $\sA_{\ell_{\log}}(\sH) = 0$, the
minimizability gap $\sM_{\ell_{\log}}(\sH)$ vanishes.  By taking the
limit on both sides, Theorem~\ref{thm:bound-log} implies the
$\sH$-consistency and Bayes-consistency of logistic loss with respect
to the top-$k$ loss. It further shows that, when the estimation error
of $\ell_{\log}$ is reduced to $\e > 0$, then the estimation error of
$\ell_{k}$ is upper bounded by $k \psi^{-1}(\e)$, which is
approximately $k \sqrt{2\e}$ for $\e$ small.

\subsection{Sum exponential loss}

In this section, we prove $\sH$-consistency bound guarantees for the
sum-exponential loss, which is defined as $\ell^{\rm{comp}}_{\exp}(h,
x, y) = \sum_{y' \neq y} e^{ h(x, y') - h(x, y) } $ and is widely
used in multi-class boosting
\citep{saberian2011multiclass,mukherjee2013theory,KuznetsovMohriSyed2014}.
\begin{restatable}{theorem}{BoundExp}
\label{thm:bound-exp}
Assume that $\sH$ is symmetric and complete. Then, for any $1 \leq k
\leq n$, the following $\sH$-consistency bound holds for the sum
exponential loss:
\ifdim\columnwidth=\textwidth
{
\begin{align*}
\sE_{\ell_k}(h) - \sE^*_{\ell_k}(\sH) + \sM_{\ell_k}(\sH) \leq k \psi^{-1} \paren*{ \sE_{\ell^{\rm{comp}}_{\exp}}(h)
    - \sE^*_{\ell^{\rm{comp}}_{\exp}}(\sH) + \sM_{\ell^{\rm{comp}}_{\exp}}(\sH) },
\end{align*}
}
\else{
\begin{align*}
& \sE_{\ell_k}(h) - \sE^*_{\ell_k}(\sH) + \sM_{\ell_k}(\sH)\\
  & \quad \leq k \psi^{-1} \paren*{ \sE_{\ell^{\rm{comp}}_{\exp}}(h)
    - \sE^*_{\ell^{\rm{comp}}_{\exp}}(\sH) + \sM_{\ell^{\rm{comp}}_{\exp}}(\sH) },
\end{align*}
}
\fi
where $\psi(t) = 1 - \sqrt{1 - t^2}$, $t \in [0,1]$. In the special
case where $\sA_{\ell^{\rm{comp}}_{\exp}}(\sH) = 0$, for any $1 \leq k
\leq n$, the following bound holds:
\begin{align*}
  \sE_{\ell_k}(h) - \sE^*_{\ell_k}(\sH)
  \leq k \psi^{-1} \paren*{ \sE_{\ell^{\rm{comp}}_{\exp}}(h)
    - \sE^*_{\ell^{\rm{comp}}_{\exp}}(\sH)},
\end{align*}
\end{restatable}
The proof is included in Appendix~\ref{app:bound-exp}. The second part
follows from the fact that when $\sA_{\ell^{\rm{comp}}_{\exp}}(\sH) =
0$, the minimizability gap $\sM_{\ell^{\rm{comp}}_{\exp}}(\sH)$
vanishes.  As with the logistic loss, the sum exponential loss is
Bayes-consistent and $\sH$-consistent with respect to the top-$k$
loss. Here too, when the estimation error of $\ell^{\rm{comp}}_{\exp}$
is reduced to $\e$, the estimation error of $\ell_{k}$ is upper
bounded by $k \psi^{-1}(\e) \approx k \sqrt{2\e}$ for sufficiently
small $\e > 0$.

\subsection{Mean absolute error loss}

The mean absolute error loss, defined as $\ell_{\rm{mae}}(h, x, y) = 1
- \bracket*{\sum_{y' \in \sY}e^{ h(x, y') - h(x, y) }}^{-1}$, is known
to be robust to label noise for training neural networks
\citep{ghosh2017robust}. The following shows that it benefits from
$\sH$-consistency bounds with respect to the top-$k$ loss as well.
\begin{restatable}{theorem}{BoundMAE}
\label{thm:bound-mae}
Assume that $\sH$ is symmetric and complete. Then, for any $1 \leq k
\leq n$, the following $\sH$-consistency bound holds for the mean
absolute error loss:
\ifdim\columnwidth=\textwidth
{
\begin{align*}
\sE_{\ell_k}(h) - \sE^*_{\ell_k}(\sH) + \sM_{\ell_k}(\sH) \leq k n\paren*{ \sE_{\ell_{\rm{mae}}}(h) - \sE^*_{\ell_{\rm{mae}}}(\sH) + \sM_{\ell_{\rm{mae}}}(\sH) }.
\end{align*}
}
\else{
\begin{align*}
& \sE_{\ell_k}(h) - \sE^*_{\ell_k}(\sH) + \sM_{\ell_k}(\sH)\\
&\quad \leq k n\paren*{ \sE_{\ell_{\rm{mae}}}(h) - \sE^*_{\ell_{\rm{mae}}}(\sH) + \sM_{\ell_{\rm{mae}}}(\sH) }.
\end{align*}
}
\fi
In the special case where $\sA_{\rm{mae}}(\sH) = 0$, for any $1 \leq k \leq n$, the following bound holds:
\begin{align*}
\sE_{\ell_k}(h) - \sE^*_{\ell_k}(\sH) \leq k n\paren*{ \sE_{\ell_{\rm{mae}}}(h) - \sE^*_{\ell_{\rm{mae}}}(\sH) }.
\end{align*}
\end{restatable}
The proof is included in Appendix~\ref{app:bound-mae}. The second part
follows from the fact that when $\sA_{\ell_{\rm{mae}}}(\sH) = 0$, the
minimizability gap $\sM_{\ell_{\rm{mae}}}(\sH)$ vanishes.  As for the
logistic loss and the sum exponential loss, the result implies
Bayes-consistency.  However, different from these losses, the bound
for the mean absolute error loss is only linear: when the estimation
error of $\ell_{\rm{\e}}$ is reduced to $\e$, the estimation error of
$\ell_{k}$ is upper bounded by $k n \e$.  The downside of this more
favorable linear rate is the dependency in the number of classes and
the fact that the mean absolute value loss is harder to optimize
\cite{zhang2018generalized}.

\subsection{Generalized cross-entropy loss}

Here, we provide $\sH$-consistency bounds for the generalized
cross-entropy loss, which is defined as $\ell_{\rm{gce}}(h, x, y) =
\frac{1}{\alpha}\bracket*{1 - \bracket*{\sum_{y' \in \sY}e^{ h(x, y')
      - h(x, y) }}^{-\alpha}}$, $\alpha \in (0, 1)$, and is a
generalization of the logistic loss and mean absolute error loss for
learning deep neural networks with noisy labels
\citep{zhang2018generalized}.

\begin{restatable}{theorem}{BoundGCE}
\label{thm:bound-gce}
Assume that $\sH$ is symmetric and complete. Then, for any $1 \leq k
\leq n$, the following $\sH$-consistency bound holds for the
generalized cross-entropy:
\ifdim\columnwidth=\textwidth
{
\begin{align*}
\sE_{\ell_k}(h) - \sE^*_{\ell_k}(\sH) + \sM_{\ell_k}(\sH) \leq k \psi^{-1} \paren*{ \sE_{\ell_{\rm{gce}}}(h) - \sE^*_{\ell_{\rm{gce}}}(\sH) + \sM_{\ell_{\rm{gce}}}(\sH) },
\end{align*}
}
\else{
\begin{align*}
& \sE_{\ell_k}(h) - \sE^*_{\ell_k}(\sH) + \sM_{\ell_k}(\sH)\\
&\quad \leq k \psi^{-1} \paren*{ \sE_{\ell_{\rm{gce}}}(h) - \sE^*_{\ell_{\rm{gce}}}(\sH) + \sM_{\ell_{\rm{gce}}}(\sH) },
\end{align*}
}
\fi
where $\psi(t) = \frac{1}{\alpha n^{\alpha}}
\bracket*{\bracket*{\frac{\paren*{1 + t}^{\frac1{1 - \alpha }} +
      \paren*{1 - t}^{\frac1{1 - \alpha }}}{2}}^{1 - \alpha } -1}$,
for all $\alpha \in (0,1)$, $t \in [0, 1]$. In the special case where
$\sA_{\ell_{\rm{gce}}}(\sH) = 0$, for any $1 \leq k \leq n$, the
following upper bound holds:
\begin{align*}
  \sE_{\ell_k}(h) - \sE^*_{\ell_k}(\sH)
  \leq k \psi^{-1} \paren*{ \sE_{\ell_{\rm{gce}}}(h)
    - \sE^*_{\ell_{\rm{gce}}}(\sH) },
\end{align*}
\end{restatable}
The proof is presented in Appendix~\ref{app:bound-gce}. The second
part follows from the fact that when $\sA_{\ell_{\rm{gce}}}(\sH) = 0$,
the minimizability gap $\sM_{\ell_{\rm{gce}}}(\sH)$ vanishes.  The
bound for the generalized cross-entropy loss depends on both the
number of classes $n$ and the parameter $\alpha$. When the estimation
error of $\ell_{\log}$ is reduced to $\e$, the estimation error of
$\ell_{k}$ is upper bounded by $k \psi^{-1}(\e) \approx k \sqrt{2
  n^{\alpha} \e}$ for sufficiently small $\e > 0$. A by-product of
this result is the Bayes-consistency of generalized cross-entropy.

In the proof of previous sections, we used the fact that the conditional regret of the top-$k$ loss is the sum of $k$ differences between two probabilities. We then upper bounded each difference with the conditional regret of the comp-sum loss, using a hypothesis based on the two probabilities. The final bound is derived by summing these differences.

\subsection{Minimizability gaps and realizability}

The key quantities in our $\sH$-consistency bounds are the
minimizability gaps, which can be upper bounded by the approximation
error, or more refined terms, depending on the magnitude of the
parameter space, as discussed by \citet{mao2023cross}. As pointed out
by these authors, these quantities, along with the functional form,
can help compare different comp-sum loss functions.

Here, we further discuss the important role of minimizability gaps
under the realizability assumption, and the connection with some
negative results of \citet{yang2020consistency}.

\begin{definition}[\textbf{top-$k$-$\sH$-realizability}]
\label{def:rel}
A distribution $\sD$ over $\sX\times\sY$ is
\emph{top-$k$-$\sH$-realizable}, if there exists a hypothesis $h\in
\sH$ such that $\mathbb{P}_{(x,y)\sim \sD}\paren*{h(x, y) > h(x,
  \hh_{k + 1}(x))} = 1$.
\end{definition}
This extends the $\sH$-realizability definition from standard
(top-$1$) classification \citep{long2013consistency} to top-$k$
classification for any $k \geq 1$.
\begin{definition}
  We say that a hypothesis set \emph{$\sH$ is closed under scaling},
  if it is a cone, that is for all $h \in \sH$ and $\alpha \in
  \Rset_{+}$, $\alpha h \in \sH$.
\end{definition}
\begin{definition}
We say that a surrogate loss $\ell$ is \emph{realizable
$\sH$-consistent with respect to $\ell_k$}, if for all $k \in [1, n]$,
and for any sequence of hypotheses $\curl*{h_n}_{n \in \Nset} \subset
\sH$ and top-$k$-$\sH$-realizable distribution, $\lim_{n \to \plus
  \infty} \sE_{\ell}\paren*{h_n} - \sE^*_{\ell}\paren*{\sH} = 0$
implies $\lim_{n \to \plus \infty} \sE_{\ell_{k}}\paren*{h_n} -
\sE^*_{\ell_{k}}\paren*{\sH} = 0$.
\end{definition}
When $\sH$ is closed under scaling, for $k = 1$ and all comp-sum loss
functions $\ell = \ell_{\rm{log}}$, $\ell^{\rm{comp}}_{\exp}$,
$\ell_{\rm{gce}}$ and $\ell_{\rm{mae}}$, it can be shown that
$\sE^*_{\ell}(\sH) = \sM_{\ell}(\sH) = 0$ for any $\sH$-realizable
distribution. For example, for $\ell = \ell_{\rm{log}}$, by using the
Lebesgue dominated convergence theorem,
\begin{align*}
  \sM_{\ell_{\log}}(\sH) &\leq \sE^*_{\ell_{\log}}(\sH)
  \leq \lim_{\beta \to \plus \infty} \sE_{\ell_{\log}}(\beta h^*)\\
& = \lim_{\beta \to \plus \infty} \log \bracket[\bigg]{1 + \sum_{y' \neq y} e^{\beta \paren*{ h^*(x, y') - h^*(x, y)}}} = 0
\end{align*}
where $h^*$ satisfies $\mathbb{P}_{(x,y)\sim \sD}\paren*{h^*(x, y) > h^*(x, \hh_{2}(x))} = 1$
Therefore, Theorems~\ref{thm:bound-log}, \ref{thm:bound-exp}, \ref{thm:bound-mae} and \ref{thm:bound-gce} imply that all these loss
functions are realizable $\sH$-consistent with respect to $\ell_{0-1}$
($\ell_k$ for $k = 1$) when $\sH$ is closed under scaling.
\begin{restatable}{theorem}{RealizabilityP}
\label{thm:realizability-p}
Assume that $\sH$ is closed under scaling. Then, $\ell_{\rm{log}}$,
$\ell^{\rm{comp}}_{\exp}$, $\ell_{\rm{gce}}$ and $\ell_{\rm{mae}}$ are
realizable $\sH$-consistent with respect to $\ell_{0-1}$.
\end{restatable}
The formal proof is presented in
Appendix~\ref{app:realizability}. However, for $ k > 1 $, since in the
realizability assumption, $h(x, y)$ is only larger than $h(x, \hh_{k +
  1}(x))$ and can be smaller than $h(x, \hh_{1}(x))$, there may exist an
$\sH$-realizable distribution $\sD$ such that $\sM_{\ell_{\log}}(\sH)
> 0$. This explains the inconsistency of the logistic loss
on top-$k$ separable data with linear predictors, when $k = 2$ and $n
> 2$, as shown in \citep{yang2020consistency}. More generally, the exact same example in \citep[Proposition~5.1]{yang2020consistency} can be used to show that all the comp-sum losses, $\ell_{\rm{log}}$,
$\ell^{\rm{comp}}_{\exp}$, $\ell_{\rm{gce}}$ and $\ell_{\rm{mae}}$ are not realizable $\sH$-consistent with
respect to $\ell_k$. Nevertheless, as previously shown, when the
hypothesis set $\sH$ adopted is sufficiently rich such that
$\sM_{\ell}(\sH) = 0$ or even $\sA_{\ell}(\sH) = 0$, they are
guaranteed to be $\sH$-consistent. This is typically the case in
practice when using deep neural networks.

\section{$\sH$-Consistency Bounds for Constrained Losses}
\label{sec:cstnd}

Constrained losses are defined as a summation of a function
$\Phi$ applied to the scores, subject to a constraint, as shown in
\citep{lee2004multicategory,awasthi2022multi}. For any $h \in \sH$ and
$ (x, y) \in \sX \times \sY$, they are expressed as
\begin{align*}
\ell^{\rm{cstnd}}(h, x, y)
= \sum_{y'\neq y} \Phi\paren*{-h(x, y')},
\end{align*}
with the constraint $\sum_{y\in \sY} h(x,y) = 0$, where $\Phi \colon
\Rset \to \Rset_{+} $ is non-increasing. When $\Phi$ is chosen as the
function $t \mapsto e^{-t}$, $t \mapsto \max \curl*{0, 1 - t}^2$, $t
\mapsto \max \curl*{0, 1 - t}$ and $t \mapsto \min\curl*{\max\curl*{0,
    1 - t/\rho}, 1}$, $\rho > 0$, $\ell^{\rm{cstnd}}(h, x, y)$ are
referred to as the constrained exponential loss
$\ell^{\rm{cstnd}}_{\rm{exp}}$, the constrained squared hinge loss
$\ell_{\rm{sq-hinge}}$, the constrained hinge loss
$\ell_{\rm{hinge}}$, and the constrained $\rho$-margin loss
$\ell_{\rho}$, respectively \citep{awasthi2022multi}. We
now study these loss functions and show that they benefit
from $\sH$-consistency bounds with respect to the top-$k$ loss. 
% Similar to the comp-sum losses, we constructively used the fact that the conditional regret of the top-$k$ loss can be expressed as the sum of $k$ differences between two conditional probabilities.

% Similar to the comp-sum losses, we used the fact that the conditional regret of the top-$k$ loss is the sum of $k$ differences between two probabilities. We then upper bounded each difference with the conditional regret of the comp-sum loss, using a hypothesis based on the two probabilities. However, the key difference lies in the fact that the selected hypothesis assumes a different form compared to the comp-sum loss case.

\subsection{Constrained exponential loss}

We first consider the constrained exponential loss, defined as
$\ell^{\rm{cstnd}}_{\rm{exp}}(h, x, y) = \sum_{y'\neq y} e^{h(x,
  y')}$. The following result provide $\sH$-consistency bounds for
$\ell^{\rm{cstnd}}_{\rm{exp}}$.
\begin{restatable}{theorem}{BoundExpCstnd}
\label{thm:bound-exp-cstnd}
Assume that $\sH$ is symmetric and complete. Then, for any $1 \leq k \leq n$, the following $\sH$-consistency bound holds for the constrained exponential loss:
\ifdim\columnwidth=\textwidth
{
\begin{align*}
\sE_{\ell_k}(h) - \sE^*_{\ell_k}(\sH) + \sM_{\ell_k}(\sH) \leq 2k \, \paren*{ \sE_{\ell^{\rm{cstnd}}_{\exp}}(h)
    - \sE^*_{\ell^{\rm{cstnd}}_{\exp}}(\sH)
    + \sM_{\ell^{\rm{cstnd}}_{\exp}}(\sH) }^{\frac{1}{2}}.
\end{align*}
}
\else{
\begin{align*}
& \sE_{\ell_k}(h) - \sE^*_{\ell_k}(\sH) + \sM_{\ell_k}(\sH)\\
  & \quad \leq 2k\, \paren*{ \sE_{\ell^{\rm{cstnd}}_{\exp}}(h)
    - \sE^*_{\ell^{\rm{cstnd}}_{\exp}}(\sH)
    + \sM_{\ell^{\rm{cstnd}}_{\exp}}(\sH) }^{\frac{1}{2}}.
\end{align*}
}
\fi
In the special case where $\sA_{\ell^{\rm{cstnd}}_{\exp}}(\sH) = 0$,
for any $1 \leq k \leq n$, the following bound holds:
\begin{align*}
  \sE_{\ell_k}(h) - \sE^*_{\ell_k}(\sH) \leq 2k\,
  \paren*{ \sE_{\ell^{\rm{cstnd}}_{\exp}}(h)
    - \sE^*_{\ell^{\rm{cstnd}}_{\exp}}(\sH)}^{\frac{1}{2}}.
\end{align*}
\end{restatable}
The proof is included in Appendix~\ref{app:bound-exp-cstnd}. The
second part follows from the fact that when
$\sA_{\ell^{\rm{cstnd}}_{\exp}}(\sH) = 0$, we have
$\sM_{\ell^{\rm{cstnd}}_{\exp}}(\sH) = 0$.  Therefore, the constrained
exponential loss is $\sH$-consistent and Bayes-consistent with respect
to $\ell_k$. If the surrogate estimation error
$\sE_{\ell^{\rm{cstnd}}_{\exp}}(h) -
\sE^*_{\ell^{\rm{cstnd}}_{\exp}}(\sH)$ is $\e$, then, the target
estimation error satisfies $\sE_{\ell_k}(h) - \sE^*_{\ell_k}(\sH) \leq
2k \sqrt{\e}$.

\subsection{Constrained squared hinge loss}

Here, we consider the constrained squared hinge loss, defined as
$\ell_{\rm{hinge}}(h, x, y) = \sum_{y'\neq y} \max \curl*{0, 1 + h(x,
  y')}^2$.  The following result shows that $\ell_{\rm{sq-hinge}}$
admits an $\sH$-consistency bound with respect to $\ell_k$.
\begin{restatable}{theorem}{BoundSqHinge}
\label{thm:bound-sq-hinge}
Assume that $\sH$ is symmetric and complete. Then, for any $1 \leq k
\leq n$, the following $\sH$-consistency bound holds for the
constrained squared hinge loss:
\ifdim\columnwidth=\textwidth
{
\begin{align*}
\sE_{\ell_k}(h) - \sE^*_{\ell_k}(\sH) + \sM_{\ell_k}(\sH) \leq 2k\, \paren*{ \sE_{\ell_{\rm{sq-hinge}}}(h) - \sE^*_{\ell_{\rm{sq-hinge}}}(\sH) + \sM_{\ell_{\rm{sq-hinge}}}(\sH) }^{\frac{1}{2}}.
\end{align*}
}
\else{
\begin{align*}
& \sE_{\ell_k}(h) - \sE^*_{\ell_k}(\sH) + \sM_{\ell_k}(\sH)\\
& \leq 2k\, \paren*{ \sE_{\ell_{\rm{sq-hinge}}}(h) - \sE^*_{\ell_{\rm{sq-hinge}}}(\sH) + \sM_{\ell_{\rm{sq-hinge}}}(\sH) }^{\frac{1}{2}}.
\end{align*}
}
\fi
In the special case where $\sA_{\ell_{\rm{sq-hinge}}}(\sH) = 0$, for any $1 \leq k \leq n$, the following bound holds:
\begin{align*}
\sE_{\ell_k}(h) - \sE^*_{\ell_k}(\sH) \leq 2k\, \paren*{ \sE_{\ell_{\rm{sq-hinge}}}(h) - \sE^*_{\ell_{\rm{sq-hinge}}}(\sH) }^{\frac{1}{2}}.
\end{align*}
\end{restatable}
The proof is included in Appendix~\ref{app:bound-sq-hinge}. The second
part follows from the fact that when the hypothesis set $\sH$ is
sufficiently rich such that $\sA_{\ell_{\rm{sq-hinge}}}(\sH) = 0$, we
have $\sM_{\ell_{\rm{sq-hinge}}}(\sH) = 0$.  As with the constrained
exponential loss, the bound is square root:
$\sE_{\ell_{\rm{sq-hinge}}}(h) - \sE^*_{\ell_{\rm{sq-hinge}}}(\sH)
\leq \e \Rightarrow \sE_{\ell_k}(h) - \sE^*_{\ell_k}(\sH) \leq 2k \,
\sqrt{\e}$. This also implies that $\ell_{\rm{sq-hinge}}$ is
Bayes-consistent with respect to $\ell_k$.

\subsection{Constrained hinge loss and $\rho$-margin loss}
Similarly, in Appendix~\ref{app:bound-hinge} and \ref{app:bound-rho},  we study the constrained hinge loss and the constrained $\rho$-margin loss, respectively. Both are shown to admit a linear $\sH$-consistency bound and are Bayes-consistent with respect to $\ell_k$ (See Theorems~\ref{thm:bound-hinge} and \ref{thm:bound-rho})

\section{Cardinality-Aware Loss Functions}
\label{sec:cardinality}

The strong theoretical results of the previous sections demonstrate
that for common hypothesis sets used in practice, comp-sum losses and
constrained losses can be effectively used as surrogate losses for the
target top-$k$ loss.  Nonetheless, the algorithms seeking to minimize
these surrogate losses offer no guidance on the crucial task of
determining the optimal cardinality $k$ for top-$k$ classification
applications. This selection is essential for practical performance,
as it directly influences the number of predicted positives. 

In this section, our goal is to select a suitable top-$k$ classifier for each input instance $x$. For easier input instances, the top-$k$ set with a smaller $k$ contains the accurate label, while it may be necessary to resort to larger $k$ values for harder input instances. Choosing $k$ optimally for each instance allows us to maintain accuracy while reducing the average cardinality used.

To tackle this problem, we introduce target cardinality-aware loss
functions for top-$k$ classification through instance-dependent
cost-sensitive learning. Then, we propose two novel families of
instance-dependant cost-sensitive surrogate losses. These loss
functions are derived by augmenting the standard comp-sum losses and
constrained loss with the corresponding cost. We show the benefits of
these surrogate losses by proving that they admit $\sH$-consistency
bounds with respect to the target cardinality-aware loss
functions. Minimizing these loss functions leads to a family of new
cardinality-aware algorithms for top-$k$ classification.

\subsection{Instance-Dependent Cost-Sensitive Learning}
\label{sec:cost-learning}

Given a pre-fixed subset $\sK = \curl*{k_1, \ldots, k_m} \subset [n] $
of all possible choices for cardinality $k$, our goal is to select the
best $k$ in the sample such that the top-$k$ loss is minimized while
using a small cardinality. More precisely, let $c \colon \sX \times
\sK \times \sY$ be a instance-dependent cost function, defined as
\begin{equation}
\label{eq:cost}
\begin{aligned}
c(x, k, y) 
& = \ell_{k}(h, x, y) + \lambda \sC(k)\\
& = 1_{y \notin \curl*{\hh_1(x), \ldots, \hh_{k}(x)}} + \lambda \sC(k)
\end{aligned}    
\end{equation}
for some function $\sC \colon [n] \to \Rset_{+}$ and parameter $\lambda > 0$. Let $\sR$ be a hypothesis set of functions mapping from $\sX \times \sK$ to $\Rset$. The prediction of a cardinality selector $r \in \sR$ is defined as the cardinality corresponding to the highest score, that is $\rr(x) = \argmax_{k \in \sK} r(x, k)$. In the event of a tie for the highest score, the cardinality $\rr(x)$ is selected based on the highest index when considering the natural order of labels.

Then, our target cardinality aware loss function $\wt \ell$ can be defined as follows: for all $r \in \sR$, $x \in \sX$ and $y \in \sY$,
\begin{equation}
\label{eq:target-cardinality}
\wt \ell (r, x, y) = c(x, \rr(x), y).
\end{equation}
For example, when the function $\sC$ is chosen as $t \colon \mapsto
\log(t)$, the learner will select a cardinality selector $r \in \sR$
that selects the best $k$ among $\sK$ for each instance $x$, in terms
of balancing the top-$k$ loss with the magnitude of $\log(k)$.

Note that our work focuses on determining the optimal cardinality $k$ for top-$k$ classification, and thus the cost function defined in \eqref{eq:cost} is based on the top-$k$ sets. However, it can potentially be generalized to other settings, such as those described in \citep{denis2017confidence}, by using confidence sets and learning a model $r$ to select the optimal confidence set based on the instance.

\eqref{eq:target-cardinality} is an instance-dependent cost-sensitive
learning problem. However, directly minimizing this target loss is
intractable. In the next sections, we will propose novel surrogate
losses to address this problem. As a useful tool, we characterized the conditional regret of the target cardinality-aware loss function in Lemma~\ref{lemma:regret-target-cost}, which can be found in Appendix~\ref{app:cost}.

Without loss of generality, assume that $0 \leq c(x, k, y) \leq 1$,
which can be achieved by normalizing the cost function.

\subsection{Cost-Sensitive Comp-Sum Losses}
\label{sec:cost-comp}
 We first
introduce a new family of surrogate losses, that we called
\emph{cost-sensitive comp-sum losses}. They are defined as follows:
for all $(r, x, y) \in \sR \times \sX \times \sY$:
\begin{align*}
\wt \ell^{\rm{comp}}(r, x, y) = \sum_{k \in \sK} \paren*{1 - c(x, k, y)} \ell^{\rm{comp}}(r, x, k).
\end{align*}
For example, when $\ell^{\rm{comp}} = \ell_{\log}$, we obtain the
cost-sensitive logistic loss as follows:
\begin{align}
\label{eq:cost-log}
& \wt \ell_{\log}(r, x, y) \nonumber\\
& = \sum_{k \in \sK} \paren*{1 - c(x, k, y)} \ell_{\log}(r, x, k) \nonumber\\
& = \sum_{k \in \sK} \paren*{1 - c(x, k, y)} \log \paren*{\sum_{k' \in \sK}e^{ r(x, k') - r(x, k) }}. 
\end{align}
Similarly, we will use $\wt \ell^{\rm{comp}}_{\exp}$, $\wt \ell_{\rm{gce}}$
and $\wt \ell_{\rm{mae}}$ to denote the corresponding cost-sensitive
counterparts for the sum-exponential loss, generalized cross-entropy
loss and mean absolute error loss, respectively.  Next, we show that
these cost-sensitive surrogate loss functions benefit from
$\sR$-consistency bounds with respect to the target loss $\wt \ell$.
\begin{restatable}{theorem}{BoundCostComp}
\label{thm:bound-cost-comp}
Assume that $\sR$ is symmetric and complete. Then, the following
$\sR$-consistency bound holds for the cost-sensitive comp-sum loss:
\ifdim\columnwidth=\textwidth
{
\begin{align*}
\sE_{\wt \ell}(r) - \sE^*_{\wt \ell}(\sR) + \sM_{\wt \ell}(\sR) \leq \gamma \paren*{ \sE_{\wt \ell^{\rm{comp}}}(r)
    - \sE^*_{\wt \ell^{\rm{comp}}}(\sR) + \sM_{\wt \ell^{\rm{comp}}}(\sR) };
\end{align*}
}
\else{
\begin{align*}
& \sE_{\wt \ell}(r) - \sE^*_{\wt \ell}(\sR) + \sM_{\wt \ell}(\sR)\\
  & \quad \leq \gamma \paren*{ \sE_{\wt \ell^{\rm{comp}}}(r)
    - \sE^*_{\wt \ell^{\rm{comp}}}(\sR) + \sM_{\wt \ell^{\rm{comp}}}(\sR) };
\end{align*}
}
\fi
In the special case where $\sR = \sR_{\rm{all}}$, the following holds:
\begin{align*}
  \sE_{\wt \ell}(r) - \sE^*_{\wt \ell}(\sR_{\rm{all}})
  \leq \gamma \paren*{ \sE_{\wt \ell^{\rm{comp}}}(r)
    - \sE^*_{\wt \ell^{\rm{comp}}}(\sR_{\rm{all}})},
\end{align*}
where $\gamma(t) = 2\sqrt{t}$ when
$\wt \ell^{\mathrm{comp}}$ is either $\wt \ell_{\rm{log}}$ or
$\wt \ell_{\rm{exp}}^{\mathrm{comp}}$; $\gamma(t) = 2\sqrt{ n^{\alpha} t}$ when
$\wt \ell^{\mathrm{comp}}$ is $\wt \ell_{\rm{gce}}$; and
$\gamma(t) = nt$ when
$\wt \ell^{\mathrm{comp}}$ is $\wt \ell_{\rm{mae}}$.
\end{restatable}
The proof is included in Appendix~\ref{app:bound-cost-comp}. The
second part follows from the fact that when $\sR = \sR_{\rm{all}}$,
all the minimizability gaps vanish. In particular,
Theorem~\ref{thm:bound-cost-comp} implies the Bayes-consistency of
cost-sensitive comp-sum losses. The bounds for cost-sensitive
generalized cross-entropy and mean absolute error loss depend on the
number of classes, making them less favorable when $n$ is large. As pointed out earlier, while the cost-sensitive mean absolute error
loss admits a linear rate, it is difficult to optimize even in the
standard classification, as reported by \citet{zhang2018generalized}
and \citet{mao2023cross}. 

In the proof, we represented the comp-sum loss as a function of the softmax and introduced a softmax-dependent function $\sS_{\mu}$ to upper bound the conditional regret of the target cardinality-aware loss function by that of the cost-sensitive comp-sum loss. This technique is novel and differs from the approach used in the standard scenario (Section~\ref{sec:comp}).

\subsection{Cost-Sensitive Constrained Losses}
\label{sec:cost-cstnd}

Motivated by the formulation of constrained loss functions in the
standard multi-class classification, we introduce a new family of
surrogate losses, termed \emph{cost-sensitive constrained
losses}, which are defined, for all $(r, x, y) \in \sR
\times \sX \times \sY$, by
\begin{align*}
  \wt \ell^{\rm{cstnd}}(r, x, y)
  = \sum_{k \in \sK} c(x, k, y) \Phi\paren*{-r(x, k)},
\end{align*}
with the constraint that $\sum_{y\in \sY} r(x,y) = 0$, where $\Phi
\colon \Rset \to \Rset_{+} $ is non-increasing.  For example, when
$\Phi(t) = e^{-t}$, we obtain the cost-sensitive constrained
exponential loss as follows:
\begin{align*}
  \wt \ell^{\rm{cstnd}}_{\exp}(r, x, y)
  = \sum_{k \in \sK} c(x, k, y) e^{r(x, k)},
\end{align*}
with the constraint that $\sum_{y\in \sY} r(x, y) = 0$. Similarly, we
will use $\wt \ell_{\rm{sq-hinge}}$, $\wt \ell_{\rm{hinge}}$ and $\wt
\ell_{\rho}$ to denote the corresponding cost-sensitive counterparts
for the constrained squared hinge loss, constrained hinge loss and
constrained $\rho$-margin loss, respectively.  Next, we show that
these cost-sensitive surrogate loss functions benefit from
$\sR$-consistency bounds with respect to the target loss $\wt \ell$.
\begin{restatable}{theorem}{BoundCostCstnd}
\label{thm:bound-cost-cstnd}
Assume that $\sR$ is symmetric and complete. Then, the following
$\sR$-consistency bound holds for the cost-sensitive constrained loss:
\ifdim\columnwidth=\textwidth
{
\begin{align*}
\sE_{\wt \ell}(r) - \sE^*_{\wt \ell}(\sR) + \sM_{\wt \ell}(\sR) \leq \gamma \paren*{ \sE_{\wt \ell^{\rm{cstnd}}}(r) - \sE^*_{\wt \ell^{\rm{cstnd}}}(\sR) + \sM_{\wt \ell^{\rm{cstnd}}}(\sR) };
\end{align*}
}
\else{
\begin{align*}
& \sE_{\wt \ell}(r) - \sE^*_{\wt \ell}(\sR) + \sM_{\wt \ell}(\sR)\\
& \quad \leq \gamma \paren*{ \sE_{\wt \ell^{\rm{cstnd}}}(r) - \sE^*_{\wt \ell^{\rm{cstnd}}}(\sR) + \sM_{\wt \ell^{\rm{cstnd}}}(\sR) };
\end{align*}
}
\fi
In the special case where $\sR = \sR_{\rm{all}}$, the following holds:
\begin{align*}
\sE_{\wt \ell}(r) - \sE^*_{\wt \ell}(\sR_{\rm{all}}) \leq \gamma \paren*{ \sE_{\wt \ell^{\rm{cstnd}}}(r) - \sE^*_{\wt \ell^{\rm{cstnd}}}(\sR_{\rm{all}})},
\end{align*}
where $\gamma(t) = 2\sqrt{t}$ when
$\wt \ell^{\mathrm{cstnd}}$ is either $\wt \ell^{\mathrm{cstnd}}_{\rm{exp}}$ or
$\wt \ell_{\rm{sq-hinge}}$; $\gamma(t) = t$ when
$\wt \ell^{\mathrm{cstnd}}$ is either $\wt \ell_{\rm{hinge}}$ or $\wt \ell_{\rho}$.
\end{restatable}
The proof is included in Appendix~\ref{app:bound-cost-cstnd}. The
second part follows from the fact that when $\sR = \sR_{\rm{all}}$,
all the minimizability gaps vanish.  In particular,
Theorem~\ref{thm:bound-cost-cstnd} implies the Bayes-consistency of
cost-sensitive constrained losses. Note that while the constrained
hinge loss and $\rho$-margin loss have a more favorable linear rate in
the bound, their optimization may be more challenging compared to
other smooth loss functions.

\begin{figure}[t]
\vskip -.1in
\begin{center}
\begin{tabular}{@{}cc@{}}
\includegraphics[scale=0.45]{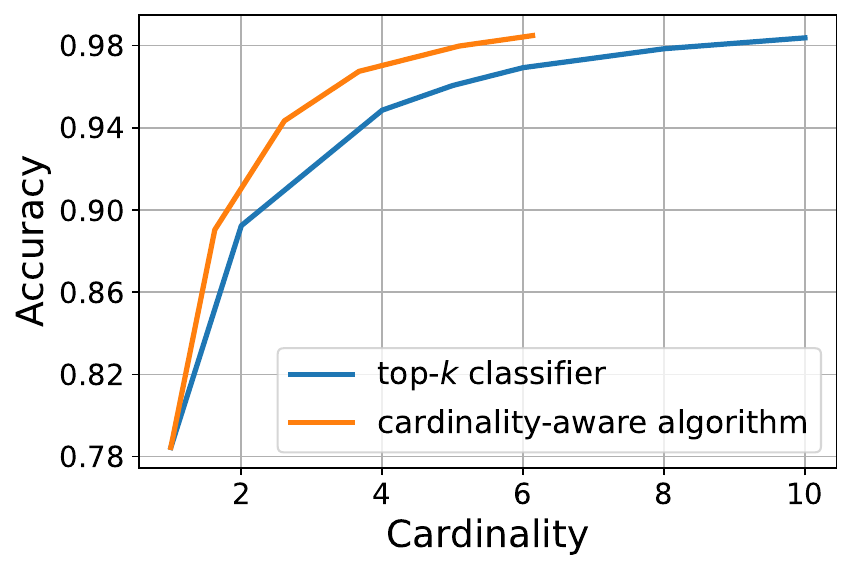}& 
\includegraphics[scale=0.45]{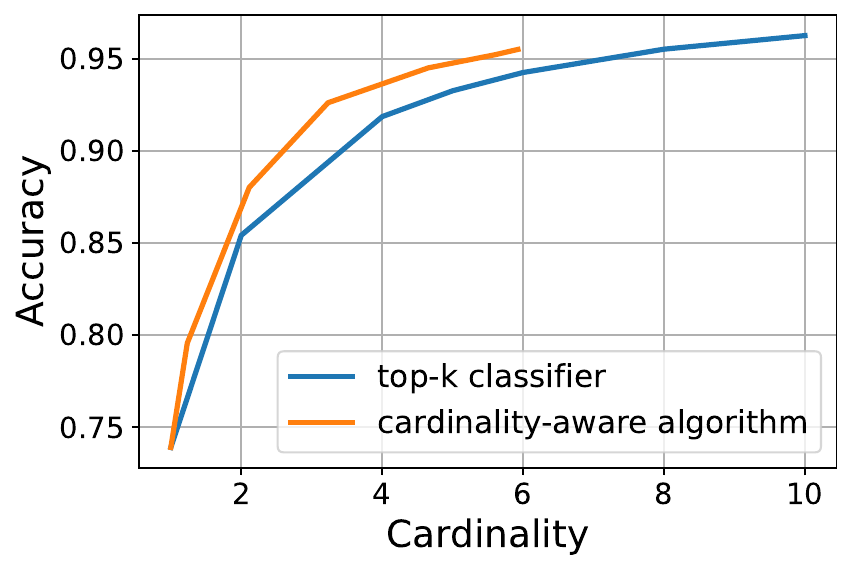}\\[-0.15cm]
{\small CIFAR-100} & {\small ImageNet} \\
\includegraphics[scale=0.45]{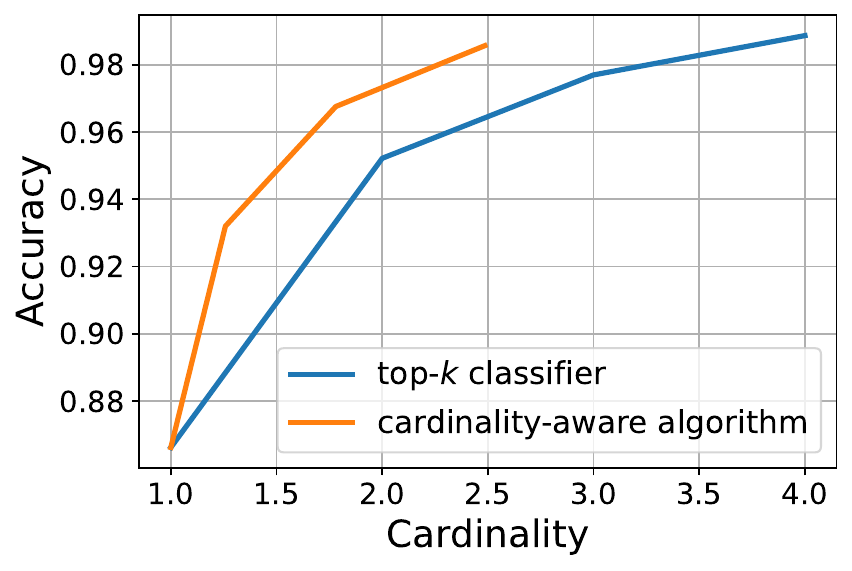}&
\includegraphics[scale=0.45]{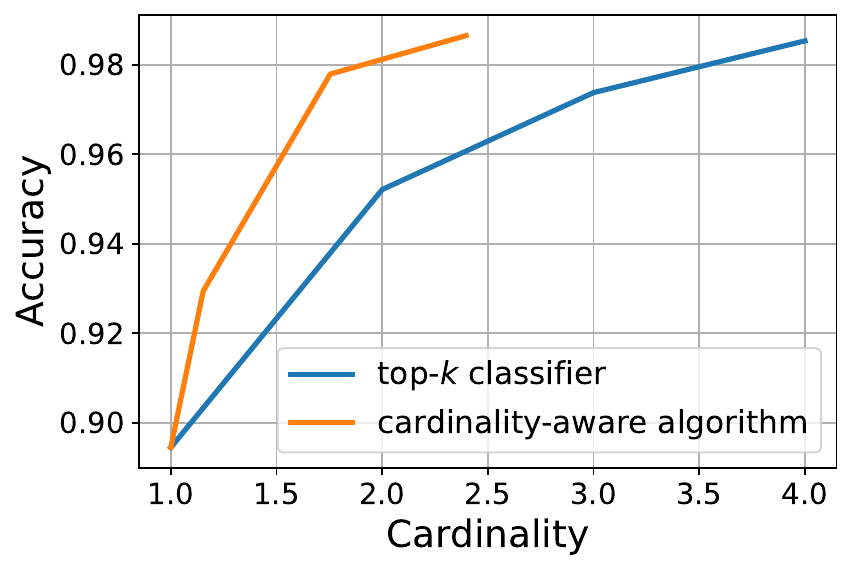}\\[-0.15cm]
{\small CIFAR-10} & {\small SVHN} 
\end{tabular}
\vskip -0.1in
\caption{Accuracy versus cardinality on various datasets.}
\label{fig:topk}
\end{center}
\vskip -0.35in
\end{figure} 

\section{Experiments}
\label{sec:experiments}
Here, we report empirical results for our cardinality-aware
algorithm and show that it consistently outperforms 
top-$k$ classifiers on benchmark datasets CIFAR-10, CIFAR-100
\citep{Krizhevsky09learningmultiple}, SVHN \citep{Netzer2011} and
ImageNet \citep{deng2009imagenet}.

We adopted a linear model for the base model $h$ to classify the
extracted features from the datasets.  We used the outputs of the
second-to-last layer of ResNet \citep{he2016deep} as features for the
CIFAR-10, CIFAR-100 and SVHN datasets. For the ImageNet dataset, we
used the CLIP \citep{radford2021learning} model to extract
features. We used a two-hidden-layer feedforward neural network
with ReLU activation functions \citep{nair2010rectified} for the
cardinality selector $r$. Both the base model $h$ and the cardinality
selector $r$ were trained using the Adam optimizer
\citep{kingma2014adam}, with a learning rate of $1\times 10^{-3}$, a
batch size of $128$, and a weight decay of $1\times 10^{-5}$.

Figure~\ref{fig:topk} compares the accuracy versus cardinality curve
of the cardinality-aware algorithm with that of top-$k$ classifiers.
The accuracy of a top-$k$ classifier is measured by $\E_{(x, y) \sim
  S}\bracket*{1 - \ell_{k}(h, x, y)}$, that is the fraction of the
sample in which the top-$k$ predictions include the true label. It
naturally grows as the cardinality $k$ increases, as shown in
Figure~\ref{fig:topk}.
The accuracy of the carnality-aware algorithms is measured by $\E_{(x,
  y) \sim S} \bracket*{1 - \ell_{\rr(x)}(h, x, y)}$, that is the fraction
of the sample in which the predictions selected by the model $r$
include the true label, and the corresponding cardinality is measured
by $\E_{(x, y) \sim S} \bracket*{\rr(x)}$, that is the average size of
the selected predictions.
The cardinality selector $r$ was trained by minimizing the
cost-sensitive logistic loss $\wt \ell_{\log}$
(Eq.~\eqref{eq:cost-log}) with the cost $c(x, k, y)$ defined as
$\ell_{k}(h, x, y) + \lambda \sC(k)$, where $\lambda = 0.05$ and $\sC(k) =
\log(k)$. We began with a set $\sK = \curl*{1}$ for the loss function
and then progressively expanded it by adding choices of larger
cardinality, each of which doubles the largest value currently in
$\sK$. In Figure~\ref{fig:topk}, the largest set $\sK$ for the
CIFAR-100 and ImageNet datasets is $\curl*{1, 2, 4, 8, 16, 32, 64}$,
whereas for the CIFAR-10 and SVHN datasets, it is $\curl*{1, 2, 4,
  8}$. As the set $\sK$ expands, there is an increase in both the
average cardinality and the accuracy.

Figure~\ref{fig:topk} shows that the cardinality-aware algorithm is
superior across the CIFAR-100, ImageNet, CIFAR-10 and SVHN
datasets. For a given cardinality $k$, the cardinality-aware algorithm
always achieves higher accuracy than a top-$k$ classifier. In other
words, to achieve the same level of accuracy, the predictions made by
the cardinality-aware algorithm can be significantly smaller in size
compared to those made by the corresponding top-$k$ classifier.
In particular, on the CIFAR-100, CIFAR-10 and SVHN datasets, the
cardinality-aware algorithm achieves the same accuracy (98\%) as the
top-$k$ classifier while using roughly only half of the
cardinality. As with the ImageNet dataset, it achieves the same
accuracy (95\%) as the top-$k$ classifier with only two-thirds of the
cardinality. This illustrates the effectiveness of our
cardinality-aware algorithm.

\section{Conclusion}

We gave a series of results demonstrating that several common
surrogate loss functions, including comp-sum losses and constrained
losses in standard classification, benefit from $\sH$-consistency
bounds with respect to the top-$k$ loss. These findings establish a
theoretical and algorithmic foundation for top-$k$ classification with
a fixed cardinality $k$.
We further introduced a cardinality-aware framework for top-$k$
classification through cost-sensitive learning, for which we proposed
cost-sensitive comp-sum losses and constrained losses that benefit
from $\sH$-consistency guarantees within this framework.
This leads to principled and practical cardinality-aware algorithms
for top-$k$ classification, which we showed empirically to be very
effective. Our analysis and algorithms are likely to be applicable
to other similar scenarios.

\bibliography{topk}

\newpage
\appendix
\onecolumn

\renewcommand{\contentsname}{Contents of Appendix}
\tableofcontents
\addtocontents{toc}{\protect\setcounter{tocdepth}{4}} 
\clearpage

\section{Proof of Lemma~\ref{lemma:regret-target}}
\label{app:regret-target}

\RegretTarget*
\begin{proof}
By definition, for any $h \in \sH$ and $x \in \sX$, the conditional
error of top-$k$ loss can be written as
\begin{equation*}
\sC_{ \ell_k }(h, x) =  \sum_{y\in \sY} p(x,y) 1_{y \notin \curl*{\hh_1(x), \ldots, \hh_k(x)}} = 1 - \sum_{i = 1}^k p(x, \hh_i(x)).
\end{equation*}
By definition of the labels $\pp_i(x)$, which are the most likely
top-$k$ labels, $\sC_{ \ell_k }(h, x)$ is minimized for $\hh_i(x) =
k_{\min}(x)$, $i \in [k]$. Since $\sH$ is regular, this choice is
realizable for some $h \in \sH$. Thus, we have
\begin{equation*}
  \sC^*_{\ell_k }(\sH, x)
  = \inf_{h \in \sH} \sC_{ \ell_k }(h, x)
  = 1 - \sum_{i = 1}^k p(x, \pp_i(x)).    
\end{equation*}
Furthermore, the calibration gap can be expressed as
\begin{align*}
\Delta\sC_{\ell_k, \sH}(h, x)  = \sC_{ \ell_k }(h, x) - \sC^*_{\ell_k}(\sH, x) = \sum_{i = 1}^k \paren*{p(x, \pp_i(x)) - p(x, \hh_i(x))},
\end{align*}
which completes the proof.
\end{proof}

\section{Proofs of \texorpdfstring{$\sH$}{H}-consistency bounds for comp-sum losses}

\subsection{Proof of Theorem~\ref{thm:bound-log}}
\label{app:bound-log}

\BoundLog*
\begin{proof}
For logistic loss $\ell_{\rm{log}}$, the conditional regret can be written as 
\begin{align*}
\Delta\sC_{\ell_{\rm{log}}, \sH}(h, x) 
& = \sum_{y = 1}^n p(x, y) \ell_{\rm{log}}(h, x, y) - \inf_{h \in \sH} \sum_{y = 1}^n p(x, y) \ell_{\rm{log}}(h, x, y)\\
& \geq \sum_{y = 1}^n p(x, y) \ell_{\rm{log}}(h, x, y) - \inf_{\mu \in \Rset} \sum_{y = 1}^n p(x, y) \ell_{\rm{log}}(h_{\mu, i}, x, y),
\end{align*}
where for any $i \in [k]$, $h_{\mu, i}(x, y) = \begin{cases}
h(x, y), & y \notin \curl*{\pp_i(x), \hh_i(x)}\\
\log\paren*{e^{h(x, \pp_i(x))} + \mu} & y = \hh_i(x)\\
\log\paren*{e^{h(x, \hh_i(x))} - \mu} & y = \pp_i(x).
\end{cases}$
Note that such a choice of $h_{\mu, i}$ leads to the following equality holds:
\begin{equation*}
\sum_{y \notin \curl*{\hh_i(x), \pp_i(x)}} p(x, y) \ell_{\rm{log}}(h, x, y) = \sum_{y \notin \curl*{\hh_i(x), \pp_i(x)}}  p(x, y) \ell_{\rm{log}}(h_{\mu, i}, x, y).
\end{equation*}
Therefore, for any $i \in [k]$, the conditional regret of logistic loss can be lower bounded as
\begin{align*}
\Delta\sC_{\ell_{\rm{log}}, \sH}(h, x)  & \geq -p(x, \hh_i(x)) \log\paren*{\frac{e^{h(x, \hh_i(x))}}{\sum_{y \in \sY} e^{h(x, y)}}} - p(x, \pp_i(x)) \log\paren*{\frac{e^{h(x, \pp_i(x))}}{\sum_{y \in \sY} e^{h(x, y)}}}\\
& \qquad + \sup_{\mu \in \Rset} \paren*{ p(x, \hh_i(x)) \log\paren*{\frac{e^{h(x, \pp_i(x))} + \mu}{\sum_{y \in \sY} e^{h(x, y)}}} + p(x, \pp_i(x)) \log\paren*{\frac{e^{h(x, \hh_i(x))} - \mu}{\sum_{y \in \sY} e^{h(x, y)}}} }\\
& = \sup_{\mu \in \Rset} \paren*{ p(x, \hh_i(x))\log\paren*{\frac{e^{h(x, \pp_i(x))} + \mu}{e^{h(x, \hh_i(x))}}} + p(x, \pp_i(x)) \log\paren*{\frac{e^{h(x, \hh_i(x))} - \mu }{e^{h(x, \pp_i(x))}}} }.
\end{align*}
By the concavity of the function, differentiate with respect to $\mu$, we obtain that the supremum is achieved by $\mu^* = \frac{p(x, \hh_i(x)) e^{h(x, \hh_i(x))} - p(x, \pp_i(x)) e^{h(x, \pp_i(x))} }{p(x, \hh_i(x)) + p(x, \pp_i(x))}$. Plug in $\mu^*$, we obtain
\begin{align*}
& \Delta\sC_{\ell_{\rm{log}}, \sH}(h, x)\\ 
& \geq p(x, \hh_i(x)) \log \paren*{\frac{p(x, \hh_i(x))}{p(x, \hh_i(x)) + p(x, \pp_i(x))} \frac{e^{h(x, \hh_i(x))} + e^{h(x, \pp_i(x))} }{e^{h(x, \hh_i(x))}}}\\
& \qquad + p(x, \pp_i(x)) \log \paren*{\frac{p(x, \pp_i(x))}{p(x, \hh_i(x)) + p(x, \pp_i(x))} \frac{e^{h(x, \hh_i(x))} + e^{h(x, \pp_i(x))} }{e^{h(x, \pp_i(x))}}}\\
& \geq p(x, \hh_i(x)) \log \paren*{\frac{2p(x, \hh_i(x))}{p(x, \hh_i(x)) + p(x, \pp_i(x))}} + p(x, \pp_i(x)) \log \paren*{\frac{2 p(x, \pp_i(x))}{p(x, \hh_i(x)) + p(x, \pp_i(x))}}
\tag{minimum is achieved when $h(x, \hh_i(x)) = h(x, \pp_i(x))$}.
\end{align*}
let $S_i = p(x, \pp_i(x)) + p(x, \hh_i(x))$ and $\Delta_i = p(x, \pp_i(x)) - p(x, \hh_i(x))$, we have
\begin{align*}
\Delta\sC_{\ell_{\rm{log}}, \sH}(h, x) 
& \geq \frac{S_i - \Delta_i}{2}\log(\frac{S_i - \Delta_i}{S_i}) + \frac{S_i + \Delta_i}{2}\log(\frac{S_i + \Delta_i}{S_i})\\
& \geq \frac{1 - \Delta_i}{2}\log(1 - \Delta_i) + \frac{1 + \Delta_i}{2}\log(1 + \Delta_i)
\tag{minimum is achieved when $S_i = 1$}\\
&  = \psi \paren*{p(x, \pp_i(x)) - p(x, \hh_i(x))},
\end{align*}
where $\psi(t) = \frac{1 - t}{2}\log(1 - t) + \frac{1 + t}{2}\log(1+ t)$, $t \in [0,1]$.
Therefore, the conditional regret of the top-$k$ loss can be upper bounded as follows:
\begin{equation*}
\Delta \sC_{\ell_k, \sH}(h, x) = \sum_{i = 1}^k \paren*{p(x, \pp_i(x)) - p(x, \hh_i(x))} \leq k \psi^{-1} \paren*{\Delta\sC_{\ell_{\rm{log}}, \sH}(h, x) }.
\end{equation*}
By the concavity of $\psi^{-1}$, take expectations on both sides of the preceding equation, we obtain
\begin{equation*}
\sE_{\ell_k}(h) - \sE^*_{\ell_k}(\sH) + \sM_{\ell_k}(\sH) \leq k \psi^{-1} \paren*{ \sE_{\ell_{\log}}(h) - \sE^*_{\ell_{\log}}(\sH) + \sM_{\ell_{\log}}(\sH) }.
\end{equation*}
The second part
follows from the fact that when $\sA_{\ell_{\log}}(\sH) = 0$, the
minimizability gap $\sM_{\ell_{\log}}(\sH)$ vanishes. 
\end{proof}

\subsection{Proof of Theorem~\ref{thm:bound-exp}}
\label{app:bound-exp}

\BoundExp*
\begin{proof}
For sum exponential loss $\ell^{\rm{comp}}_{\exp}$, the conditional regret can be written as 
\begin{align*}
\Delta\sC_{\ell^{\rm{comp}}_{\exp}, \sH}(h, x) 
& = \sum_{y = 1}^n p(x, y) \ell^{\rm{comp}}_{\exp}(h, x, y) - \inf_{h \in \sH} \sum_{y = 1}^n p(x, y) \ell^{\rm{comp}}_{\exp}(h, x, y)\\
& \geq \sum_{y = 1}^n p(x, y) \ell^{\rm{comp}}_{\exp}(h, x, y) - \inf_{\mu \in \Rset} \sum_{y = 1}^n p(x, y) \ell^{\rm{comp}}_{\exp}(h_{\mu, i}, x, y),
\end{align*}
where for any $i \in [k]$, $h_{\mu, i}(x, y) = \begin{cases}
h(x, y), & y \notin \curl*{\pp_i(x), \hh_i(x)}\\
\log\paren*{e^{h(x, \pp_i(x))} + \mu} & y = \hh_i(x)\\
\log\paren*{e^{h(x, \hh_i(x))} - \mu} & y = \pp_i(x).
\end{cases}$
Note that such a choice of $h_{\mu, i}$ leads to the following equality holds:
\begin{equation*}
\sum_{y \notin \curl*{\hh_i(x), \pp_i(x)}} p(x, y) \ell^{\rm{comp}}_{\exp}(h, x, y) = \sum_{y \notin \curl*{\hh_i(x), \pp_i(x)}}  p(x, y) \ell^{\rm{comp}}_{\exp}(h_{\mu, i}, x, y).
\end{equation*}
Therefore, for any $i \in [k]$, the conditional regret of sum exponential loss can be lower bounded as
\begin{align*}
\Delta\sC_{\ell^{\rm{comp}}_{\exp}, \sH}(h, x)  & \geq \sum_{y' \in \sY} \exp \paren*{h(x, y')}\bracket*{\frac{p(x, \hh_i(x))} {\exp \paren*{h(x,  \hh_i(x))}} + \frac{p(x, \pp_i(x))} {\exp\paren*{h(x, \pp_i(x))}}}\\
& \qquad + \sup_{\mu \in \Rset} \paren*{- \sum_{y'\in \sY} \exp\paren*{h(x, y')}\bracket*{\frac{p(x, \hh_i(x))}{ \exp\paren*{h(x, \pp_i(x))} + \mu} + \frac{p(x, \pp_i(x))}{\exp\paren*{h(x, \hh_i(x))} - \mu}} }.
\end{align*}
By the concavity of the function, differentiate with respect to $\mu$, we obtain that the supremum is achieved by $\mu^* =\frac{\exp \bracket*{h(x,\hh_i(x))}\sqrt{p(x,\hh_i(x))} - \exp\bracket*{h(x, \pp_i(x))}\sqrt{p(x, \pp_i(x))}}{\sqrt{p(x, \hh_i(x))} + \sqrt{p(x,\pp_i(x))}}$. Plug in $\mu^*$, we obtain
\begin{align*}
& \Delta\sC_{\ell^{\rm{comp}}_{\exp}, \sH}(h, x)\\ 
& \geq \sum_{y'\in \sY} \exp\paren*{h(x, y')} \bracket*{\frac{p(x, \hh_i(x))} {\exp\paren*{h(x, \hh_i(x))}} + \frac{p(x, \pp_i(x))} {\exp\paren*{h(x, \pp_i(x))}} - \frac{\paren*{\sqrt{p(x, \hh_i(x))} + \sqrt{p(x, \pp_i(x))}}^2}{\exp\paren*{h(x, \pp_i(x))} + \exp\paren*{h(x, \hh_i(x))}}}\\
&\geq \bracket*{1 + \frac{\exp\paren*{h(x, \pp_i(x))}}{\exp\paren*{h(x, \hh_i(x))}}}p(x, \hh_i(x)) + \bracket*{1 + \frac{\exp\paren*{h(x, \hh_i(x))}}{\exp\paren*{h(x, \pp_i(x))}}}p(x, \pp_i(x)) - \paren*{\sqrt{p(x, \hh_i(x))} + \sqrt{p(x, \pp_i(x))}}^2
\tag{$\sum_{y' \in \sY} \exp\paren*{h(x, y')}\geq \exp\paren*{h(x, \pp_i(x))} + \exp\paren*{h(x, \hh_i(x))}$}\\
& \geq 2 p(x,\hh_i(x)) + 2 p(x,\pp_i(x)) - \paren*{\sqrt{p(x, \hh_i(x))} + \sqrt{p(x, \pp_i(x))}}^2 \tag{minimum is attained when $\frac{\exp\paren*{h(x, \pp_i(x))}}{\exp\paren*{h(x,\hh_i(x))}} = 1$}.
\end{align*}
let $S_i = p(x, \pp_i(x)) + p(x, \hh_i(x))$ and $\Delta_i = p(x, \pp_i(x)) - p(x, \hh_i(x))$, we have
\begin{align*}
\Delta\sC_{\ell^{\rm{comp}}_{\exp}, \sH}(h, x) 
& \geq 2 S_i - \paren*{\sqrt{\frac{S_i + \Delta_i}{2}} + \sqrt{\frac{S_i - \Delta_i}{2}}}^2\\
& \geq 2\bracket*{1 -\bracket*{\frac{\paren*{1 + \Delta_i}^{\frac1{2}} + \paren*{1 - \Delta_i}^{\frac1{2}}}{2}}^{2}} 
\tag{minimum is achieved when $S_i = 1$}\\
& = 1 - \sqrt{1 - (\Delta_i)^2} \\
&  = \psi \paren*{p(x, \pp_i(x)) - p(x, \hh_i(x))},
\end{align*}
where $\psi(t) = 1 - \sqrt{1 - t^2}$, $t \in [0,1]$.
Therefore, the conditional regret of the top-$k$ loss can be upper bounded as follows:
\begin{equation*}
\Delta \sC_{\ell_k, \sH}(h, x) = \sum_{i = 1}^k \paren*{p(x, \pp_i(x)) - p(x, \hh_i(x))} \leq k \psi^{-1} \paren*{\Delta\sC_{\ell^{\rm{comp}}_{\exp}, \sH}(h, x) }.
\end{equation*}
By the concavity of $\psi^{-1}$, take expectations on both sides of the preceding equation, we obtain
\begin{equation*}
\sE_{\ell_k}(h) - \sE^*_{\ell_k}(\sH) + \sM_{\ell_k}(\sH) \leq k \psi^{-1} \paren*{ \sE_{\ell^{\rm{comp}}_{\exp}}(h) - \sE^*_{\ell^{\rm{comp}}_{\exp}}(\sH) + \sM_{\ell^{\rm{comp}}_{\exp}}(\sH) }.
\end{equation*}
The second part
follows from the fact that when $\sA_{\ell^{\rm{comp}}_{\exp}}(\sH) =
0$, the minimizability gap $\sM_{\ell^{\rm{comp}}_{\exp}}(\sH)$
vanishes.
\end{proof}

\subsection{Proof of Theorem~\ref{thm:bound-mae}}
\label{app:bound-mae}
\BoundMAE*
\begin{proof}
For mean absolute error loss $\ell_{\rm{mae}}$, the conditional regret can be written as 
\begin{align*}
\Delta\sC_{\ell_{\rm{mae}}, \sH}(h, x) 
& = \sum_{y = 1}^n p(x, y) \ell_{\rm{mae}}(h, x, y) - \inf_{h \in \sH} \sum_{y = 1}^n p(x, y) \ell_{\rm{mae}}(h, x, y)\\
& \geq \sum_{y = 1}^n p(x, y) \ell_{\rm{mae}}(h, x, y) - \inf_{\mu \in \Rset} \sum_{y = 1}^n p(x, y) \ell_{\rm{mae}}(h_{\mu, i}, x, y),
\end{align*}
where for any $i \in [k]$, $h_{\mu, i}(x, y) = \begin{cases}
h(x, y), & y \notin \curl*{\pp_i(x), \hh_i(x)}\\
\log\paren*{e^{h(x, \pp_i(x))} + \mu} & y = \hh_i(x)\\
\log\paren*{e^{h(x, \hh_i(x))} - \mu} & y = \pp_i(x).
\end{cases}$
Note that such a choice of $h_{\mu, i}$ leads to the following equality holds:
\begin{equation*}
\sum_{y \notin \curl*{\hh_i(x), \pp_i(x)}} p(x, y) \ell_{\rm{mae}}(h, x, y) = \sum_{y \notin \curl*{\hh_i(x), \pp_i(x)}}  p(x, y) \ell_{\rm{mae}}(h_{\mu, i}, x, y).
\end{equation*}
Therefore, for any $i \in [k]$, the conditional regret of mean absolute error loss can be lower bounded as
\begin{align*}
& \Delta\sC_{\ell_{\rm{mae}}, \sH}(h, x)\\  & \geq p(x,\hh_i(x)) \paren*{1-\frac{\exp\paren*{h(x,\hh_i(x))}}{\sum_{y'\in \sY}\exp\paren*{h(x,y')}}} +p(x,\pp_i(x)) \paren*{1-\frac{\exp\paren*{h(x,\pp_i(x))}}{\sum_{y'\in \sY}\exp\paren*{h(x,y')}}}\\
& \quad + \sup_{\mu \in \Rset} \paren*{-p(x,\pp_i(x)) \paren*{1-\frac{\exp\paren*{h(x,\hh_i(x))}-\mu}{\sum_{y'\in \sY}\exp\paren*{h(x,y')}}} -p(x,\hh_i(x)) \paren*{1-\frac{\exp\paren*{h(x,\pp_i(x))}+\mu}{\sum_{y'\in \sY}\exp\paren*{h(x,y')}}}}.
\end{align*}
By the concavity of the function, differentiate with respect to $\mu$, we obtain that the supremum is achieved by $\mu^* = -\exp\bracket*{h(x, \pp_i(x)}$. Plug in $\mu^*$, we obtain
\begin{align*}
& \Delta\sC_{\ell_{\rm{mae}}, \sH}(h, x)\\ 
& \geq p(x,\pp_i(x))\frac{\exp\paren*{h(x,\hh_i(x))}}{\sum_{y'\in \sY}\exp\paren*{h(x,y')}}-p(x,\hh_i(x))\frac{\exp\paren*{h(x,\hh_i(x))}}{\sum_{y'\in \sY}\exp\paren*{h(x,y')}}\\
&  \geq \frac{1}{n}
\paren*{p(x,\pp_i(x)) - p(x,\hh_i(x))}
\tag{$\frac{\exp\paren*{h(x,\hh_i(x))}}{\sum_{y'\in \sY}\exp\paren*{h(x,y')}}\geq \frac1{n}$}
\end{align*}
Therefore, the conditional regret of the top-$k$ loss can be upper bounded as follows:
\begin{equation*}
\Delta \sC_{\ell_k, \sH}(h, x) = \sum_{i = 1}^k \paren*{p(x, \pp_i(x)) - p(x, \hh_i(x))} \leq k n \paren*{\Delta\sC_{\ell_{\rm{mae}}, \sH}(h, x) }.
\end{equation*}
Take expectations on both sides of the preceding equation, we obtain
\begin{equation*}
\sE_{\ell_k}(h) - \sE^*_{\ell_k}(\sH) + \sM_{\ell_k}(\sH) \leq k n \paren*{ \sE_{\ell_{\rm{mae}}}(h) - \sE^*_{\ell_{\rm{mae}}}(\sH) + \sM_{\ell_{\rm{mae}}}(\sH) }.
\end{equation*}
The second part
follows from the fact that when $\sA_{\ell_{\rm{mae}}}(\sH) = 0$, the
minimizability gap $\sM_{\ell_{\rm{mae}}}(\sH)$ vanishes.
\end{proof}

\subsection{Proof of Theorem~\ref{thm:bound-gce}}
\label{app:bound-gce}
\BoundGCE*
\begin{proof}
For generalized cross-entropy loss $\ell_{\rm{gce}}$, the conditional regret can be written as 
\begin{align*}
& \Delta\sC_{\ell_{\rm{gce}}, \sH}(h, x)\\ 
& = \sum_{y = 1}^n p(x, y) \ell_{\rm{gce}}(h, x, y) - \inf_{h \in \sH} \sum_{y = 1}^n p(x, y) \ell_{\rm{gce}}(h, x, y)\\
& \geq \sum_{y = 1}^n p(x, y) \ell_{\rm{gce}}(h, x, y) - \inf_{\mu \in \Rset} \sum_{y = 1}^n p(x, y) \ell_{\rm{gce}}(h_{\mu, i}, x, y),
\end{align*}
where for any $i \in [k]$, $h_{\mu, i}(x, y) = \begin{cases}
h(x, y), & y \notin \curl*{\pp_i(x), \hh_i(x)}\\
\log\paren*{e^{h(x, \pp_i(x))} + \mu} & y = \hh_i(x)\\
\log\paren*{e^{h(x, \hh_i(x))} - \mu} & y = \pp_i(x).
\end{cases}$
Note that such a choice of $h_{\mu, i}$ leads to the following equality holds:
\begin{equation*}
\sum_{y \notin \curl*{\hh_i(x), \pp_i(x)}} p(x, y) \ell_{\rm{gce}}(h, x, y) = \sum_{y \notin \curl*{\hh_i(x), \pp_i(x)}}  p(x, y) \ell_{\rm{gce}}(h_{\mu, i}, x, y).
\end{equation*}
Therefore, for any $i \in [k]$, the conditional regret of generalized cross-entropy loss can be lower bounded as
\begin{align*}
& \alpha \Delta\sC_{\ell_{\rm{gce}}, \sH}(h, x)\\
& \geq  p(x, \hh_i(x)) \paren*{1 - \bracket*{\frac{\exp\paren*{h(x, \hh_i(x))}}{\sum_{y' \in \sY}\exp\paren*{h(x, y')}}}^{\alpha}} + p(x, \pp_i(x)) \paren*{1 - \bracket*{\frac{\exp\paren*{h(x, \pp_i(x))}}{\sum_{y'\in \sY}\exp\paren*{h(x, y')}}}^{\alpha}}\\
& + \sup_{\mu \in \Rset} \paren*{ -p(x, \hh_i(x)) \paren*{1 - \bracket*{\frac{\exp\paren*{h(x, \pp_i(x))} + \mu}{\sum_{y' \in \sY}\exp\paren*{h(x, y')}}}^{\alpha}} - p(x, \pp_i(x)) \paren*{1 - \bracket*{\frac{\exp\paren*{h(x, \hh_i(x))} - \mu}{\sum_{y' \in \sY}\exp\paren*{h(x, y')}}}^{\alpha}} }.
\end{align*}
By the concavity of the function, differentiate with respect to $\mu$, we obtain that the supremum is achieved by $\mu^* = \frac{\exp\bracket*{h(x,\hh_i(x))}p(x, \pp_i(x))^{\frac1{\alpha - 1}} - \exp\bracket*{h(x, \pp_i(x))}p(x,\hh_i(x))^{\frac1{\alpha - 1}}}{p(x,\hh_i(x))^{\frac1{\alpha - 1}} + p(x,\pp_i(x))^{\frac1{\alpha - 1}}}$. Plug in $\mu^*$, we obtain
\begin{align*}
& \alpha \Delta\sC_{\ell_{\rm{gce}}, \sH}(h, x)\\ 
& \geq p(x, \hh_i(x))\bracket*{\frac{\bracket*{\exp\paren*{h(x, \hh_i(x))} + \exp\paren*{h(x, \pp_i(x))}}p(x, \pp_i(x))^{\frac1{\alpha - 1}}}{\sum_{y'\in \sY}\exp\paren*{h(x, y')}\bracket*{p(x, \hh_i(x))^{\frac1{\alpha - 1}} + p(x, \pp_i(x))^{\frac1{\alpha - 1}}}}}^{\alpha } - p(x, \hh_i(x))\bracket*{\frac{\exp\paren*{h(x, \hh_i(x))}}{\sum_{y'\in \sY}\exp\paren*{h(x, y')}}}^{\alpha }\\
&\quad +p(x, \pp_i(x))\bracket*{\frac{\bracket*{\exp\paren*{h(x, \hh_i(x))} + \exp\paren*{h(x, \pp_i(x))}}p(x, \hh_i(x))^{\frac1{\alpha - 1}}}{\sum_{y'\in \sY}\exp\paren*{h(x, y')}\bracket*{p(x, \hh_i(x))^{\frac1{\alpha - 1}} + p(x, \pp_i(x))^{\frac1{\alpha - 1}}}}}^{\alpha } - p(x, \pp_i(x))\bracket*{\frac{\exp\paren*{h(x, \pp_i(x))}}{\sum_{y'\in \sY}\exp\paren*{h(x, y')}}}^{\alpha }\\
&\geq \frac{1}{n^{\alpha}}\paren*{p(x,\hh_i(x))\bracket*{\frac{2p(x,\pp_i(x))^{\frac1{\alpha - 1}}}{p(x,\hh_i(x))^{\frac1{\alpha - 1}}+p(x,\pp_i(x))^{\frac1{\alpha - 1}}}}^{\alpha}-p(x,\hh_i(x))}\\
& + \frac{1}{n^{\alpha}}\paren*{p(x,\pp_i(x))\bracket*{\frac{2p(x,\hh_i(x))^{\frac1{\alpha - 1}}}{p(x,\hh_i(x))^{\frac1{\alpha - 1}}+p(x,\pp_i(x))^{\frac1{\alpha - 1}}}}^{\alpha}-p(x,\pp_i(x))}
\tag{$\paren*{\frac{\exp\paren*{h(x,\pp_i(x))}}{\sum_{y'\in \sY}\exp\paren*{h(x,y')}}}^{\alpha }\geq \frac1{n^{\alpha }}$ and minimum is attained when $\frac{\exp\paren*{h(x,\pp_i(x))}}{\exp\paren*{h(x,\hh_i(x))}}=1$}\\
\end{align*}
let $S_i = p(x, \pp_i(x)) + p(x, \hh_i(x))$ and $\Delta_i = p(x, \pp_i(x)) - p(x, \hh_i(x))$, we have
\begin{align*}
\Delta\sC_{\ell_{\rm{gce}}, \sH}(h, x) 
& \geq \frac{1}{\alpha n^{\alpha}}\paren*{\bracket*{\frac{\paren*{S_i + \Delta_i}^{\frac1{1 - \alpha }}+\paren*{S_i - \Delta_i}^{\frac1{1 - \alpha}}}{2}}^{1 - \alpha} - S_i}\\
& \geq \frac{1}{\alpha n^{\alpha}}\paren*{\bracket*{\frac{\paren*{1 + \Delta_i}^{\frac1{1 - \alpha }}+\paren*{1 - \Delta_i}^{\frac1{1 - \alpha}}}{2}}^{1 - \alpha} - 1}\\
\tag{minimum is achieved when $S_i = 1$}\\
&  = \psi \paren*{p(x, \pp_i(x)) - p(x, \hh_i(x))},
\end{align*}
where $\psi(t) = \frac{1}{\alpha n^{\alpha}}
\bracket*{\bracket*{\frac{\paren*{1 + t}^{\frac1{1 - \alpha }} +
      \paren*{1 - t}^{\frac1{1 - \alpha }}}{2}}^{1 - \alpha } -1}$, $t \in [0,1]$.
Therefore, the conditional regret of the top-$k$ loss can be upper bounded as follows:
\begin{equation*}
\Delta \sC_{\ell_k, \sH}(h, x) = \sum_{i = 1}^k \paren*{p(x, \pp_i(x)) - p(x, \hh_i(x))} \leq k \psi^{-1} \paren*{\Delta\sC_{\ell_{\rm{gce}}, \sH}(h, x) }.
\end{equation*}
By the concavity of $\psi^{-1}$, take expectations on both sides of the preceding equation, we obtain
\begin{equation*}
\sE_{\ell_k}(h) - \sE^*_{\ell_k}(\sH) + \sM_{\ell_k}(\sH) \leq k \psi^{-1} \paren*{ \sE_{\ell_{\rm{gce}}}(h) - \sE^*_{\ell_{\rm{gce}}}(\sH) + \sM_{\ell_{\rm{gce}}}(\sH) }.
\end{equation*}
The second
part follows from the fact that when $\sA_{\ell_{\rm{gce}}}(\sH) = 0$,
the minimizability gap $\sM_{\ell_{\rm{gce}}}(\sH)$ vanishes.
\end{proof}

\section{Proofs of realizable \texorpdfstring{$\sH$}{H}-consistency for comp-sum losses}
\label{app:realizability}

\RealizabilityP*
\begin{proof}
Since the distribution is realizable, there exists a hypothesis $h \in \sH$ such that \[\mathbb{P}_{(x,y)\sim \sD}\paren*{h^*(x, y) > h^*(x, \hh_{2}(x))} = 1.\]
Therefore, for the logistic loss, by using the
Lebesgue dominated convergence theorem,
\begin{align*}
  \sM_{\ell_{\log}}(\sH) &\leq \sE^*_{\ell_{\log}}(\sH)
  \leq \lim_{\beta \to \plus \infty} \sE_{\ell_{\log}}(\beta h) = \lim_{\beta \to \plus \infty} \log \bracket[\bigg]{1 + \sum_{y' \neq y} e^{\beta \paren*{ h^*(x, y') - h^*(x, y)}}} = 0.
\end{align*}
For the sum exponential loss, by using the
Lebesgue dominated convergence theorem,
\begin{align*}
  \sM_{\ell^{\rm{comp}}_{\exp}}(\sH) &\leq \sE^*_{\ell^{\rm{comp}}_{\exp}}(\sH)
  \leq \lim_{\beta \to \plus \infty} \sE_{\ell^{\rm{comp}}_{\exp}}(\beta h) = \lim_{\beta \to \plus \infty} \sum_{y' \neq y} e^{ \beta\paren*{h^*(x, y') - h^*(x, y) }} = 0.
\end{align*}
For the generalized cross entropy loss, by using the
Lebesgue dominated convergence theorem,
\begin{align*}
  \sM_{\ell_{\rm{gce}}}(\sH) &\leq \sE^*_{\ell_{\rm{gce}}}(\sH)
  \leq \lim_{\beta \to \plus \infty} \sE_{\ell_{\rm{gce}}}(\beta h) = \lim_{\beta \to \plus \infty} \frac{1}{\alpha}\bracket*{1 - \bracket*{\sum_{y' \in \sY}e^{ \beta( h^*(x, y')
      - h^*(x, y)) }}^{-\alpha}} = 0.
\end{align*}
For the mean absolute error loss, by using the
Lebesgue dominated convergence theorem,
\begin{align*}
  \sM_{\ell_{\rm{mae}}}(\sH) &\leq \sE^*_{\ell_{\rm{mae}}}(\sH)
  \leq \lim_{\beta \to \plus \infty} \sE_{\ell_{\rm{mae}}}(\beta h) = \lim_{\beta \to \plus \infty} 1
- \bracket*{\sum_{y' \in \sY}e^{\beta \paren*{h^*(x, y') - h^*(x, y)} }}^{-1} = 0.
\end{align*}
Therefore, by Theorems~\ref{thm:bound-log}, \ref{thm:bound-exp}, \ref{thm:bound-mae} and \ref{thm:bound-gce}, the proof is completed.
\end{proof}

\section{Proofs of \texorpdfstring{$\sH$}{H}-consistency bounds for constrained losses}
\subsection{Proof of Theorem~\ref{thm:bound-exp-cstnd}}
\label{app:bound-exp-cstnd}
The conditional error for the constrained loss can be expressed as follows:
\begin{align*}
\sC_{\ell^{\rm{cstnd}}}(h, x) = \sum_{y = 1}^n p(x, y)  \ell^{\rm{cstnd}}(h, x, y) = \sum_{y = 1}^n p(x, y) \sum_{y'\neq y} \Phi\paren*{-h(x, y')} = \sum_{y\in \sY} \paren*{1-p(x, y)}\Phi\paren*{-h(x, y)}.
\end{align*}

\BoundExpCstnd*
\begin{proof}
For the constrained exponential loss $\ell^{\rm{cstnd}}_{\rm{exp}}$, the conditional regret can be written as 
\begin{align*}
\Delta\sC_{\ell^{\rm{cstnd}}_{\rm{exp}}, \sH}(h, x) 
& = \sum_{y = 1}^n p(x, y) \ell^{\rm{cstnd}}_{\rm{exp}}(h, x, y) - \inf_{h \in \sH} \sum_{y = 1}^n p(x, y) \ell^{\rm{cstnd}}_{\rm{exp}}(h, x, y)\\
& \geq \sum_{y = 1}^n p(x, y) \ell^{\rm{cstnd}}_{\rm{exp}}(h, x, y) - \inf_{\mu \in \Rset} \sum_{y = 1}^n p(x, y) \ell^{\rm{cstnd}}_{\rm{exp}}(h_{\mu, i}, x, y),
\end{align*}
where for any $i \in [k]$, $h_{\mu, i}(x, y) = \begin{cases}
h(x, y), & y \notin \curl*{\pp_i(x), \hh_i(x)}\\
h(x, \pp_i(x)) + \mu & y = \hh_i(x)\\
h(x, \hh_i(x)) - \mu & y = \pp_i(x).
\end{cases}$
Note that such a choice of $h_{\mu, i}$ leads to the following equality holds:
\begin{equation*}
\sum_{y \notin \curl*{\hh_i(x), \pp_i(x)}} p(x, y) \ell^{\rm{cstnd}}_{\rm{exp}}(h, x, y) = \sum_{y \notin \curl*{\hh_i(x), \pp_i(x)}}  p(x, y) \ell^{\rm{cstnd}}_{\rm{exp}}(h_{\mu, i}, x, y).
\end{equation*}
Let $q(x, \pp_i(x)) = 1 - p(x, \pp_i(x))$ and $q(x, \hh_i(x)) = 1 - p(x, \hh_i(x))$.
Therefore, for any $i \in [k]$, the conditional regret of constrained exponential loss can be lower bounded as
\begin{align*}
& \Delta\sC_{\ell^{\rm{cstnd}}_{\rm{exp}}, \sH}(h, x)\\ 
& \geq  \inf_{h \in \sH}\sup_{\mu\in \Rset} \curl*{q(x, \pp_i(x))\paren*{e^{h(x, \pp_i(x))}-e^{h(x,\hh_i(x))-\mu}}+q(x,\hh_i(x))\paren*{e^{h(x,\hh_i(x))}-e^{h(x, \pp_i(x))+\mu}}}\\
& = \paren*{\sqrt{q(x, \pp_i(x))}-\sqrt{q(x,\hh_i(x))}}^2 \tag{differentiating with respect to $\mu$, $h$ to optimize}\\
&   =   \paren*{\frac{q(x,\hh_i(x)) - q(x, \pp_i(x))}{\sqrt{q(x, \pp_i(x))} + \sqrt{q(x, \hh_i(x))}}}^2\\
& \geq \frac1{4} \paren*{q(x,\hh_i(x)) - q(x, \pp_i(x))}^2 \tag{$0\leq q(x, y)\leq 1$}\\
& = \frac1{4} \paren*{p(x,\pp_i(x)) - p(x, \hh_i(x))}^2.
\end{align*}
Therefore, by Lemma~\ref{lemma:regret-target}, the conditional regret of the top-$k$ loss can be upper bounded as follows:
\begin{equation*}
\Delta \sC_{\ell_k, \sH}(h, x) = \sum_{i = 1}^k \paren*{p(x, \pp_i(x)) - p(x, \hh_i(x))} \leq 2 k \paren*{\Delta\sC_{\ell^{\rm{cstnd}}_{\rm{exp}}, \sH}(h, x) }^{\frac12}.
\end{equation*}
By the concavity, take expectations on both sides of the preceding equation, we obtain
\begin{equation*}
\sE_{\ell_k}(h) - \sE^*_{\ell_k}(\sH) + \sM_{\ell_k}(\sH) \leq 2 k\paren*{ \sE_{\ell^{\rm{cstnd}}_{\rm{exp}}}(h) - \sE^*_{\ell^{\rm{cstnd}}_{\rm{exp}}}(\sH) + \sM_{\ell^{\rm{cstnd}}_{\rm{exp}}}(\sH) }^{\frac12}.
\end{equation*}
The second part follows from the fact that when
$\sA_{\ell^{\rm{cstnd}}_{\exp}}(\sH) = 0$, we have
$\sM_{\ell^{\rm{cstnd}}_{\exp}}(\sH) = 0$.
\end{proof}

\subsection{Proof of Theorem~\ref{thm:bound-sq-hinge}}
\label{app:bound-sq-hinge}

\BoundSqHinge*
\begin{proof}
For the constrained squared hinge loss $\ell_{\rm{sq-hinge}}$, the conditional regret can be written as 
\begin{align*}
\Delta\sC_{\ell_{\rm{sq-hinge}}, \sH}(h, x) 
& = \sum_{y = 1}^n p(x, y) \ell_{\rm{sq-hinge}}(h, x, y) - \inf_{h \in \sH} \sum_{y = 1}^n p(x, y) \ell_{\rm{sq-hinge}}(h, x, y)\\
& \geq \sum_{y = 1}^n p(x, y) \ell_{\rm{sq-hinge}}(h, x, y) - \inf_{\mu \in \Rset} \sum_{y = 1}^n p(x, y) \ell_{\rm{sq-hinge}}(h_{\mu, i}, x, y),
\end{align*}
where for any $i \in [k]$, $h_{\mu, i}(x, y) = \begin{cases}
h(x, y), & y \notin \curl*{\pp_i(x), \hh_i(x)}\\
h(x, \pp_i(x)) + \mu & y = \hh_i(x)\\
h(x, \hh_i(x)) - \mu & y = \pp_i(x).
\end{cases}$
Note that such a choice of $h_{\mu, i}$ leads to the following equality holds:
\begin{equation*}
\sum_{y \notin \curl*{\hh_i(x), \pp_i(x)}} p(x, y) \ell_{\rm{sq-hinge}}(h, x, y) = \sum_{y \notin \curl*{\hh_i(x), \pp_i(x)}}  p(x, y) \ell_{\rm{sq-hinge}}(h_{\mu, i}, x, y).
\end{equation*}
Let $q(x, \pp_i(x)) = 1 - p(x, \pp_i(x))$ and $q(x, \hh_i(x)) = 1 - p(x, \hh_i(x))$.
Therefore, for any $i \in [k]$, the conditional regret of the constrained squared hinge loss can be lower bounded as
\begin{align*}
\Delta\sC_{\ell_{\rm{sq-hinge}}, \sH}(h, x) 
&  \geq  \inf_{h \in \sH}  \sup_{\mu\in \Rset} \bigg\{q(x, \pp_i(x))\paren*{\max\curl*{0, 1 + h(x, \pp_i(x))}^2-\max\curl*{0, 1 + h(x,\hh_i(x))-\mu}^2 }\\
& \qquad + q(x,\hh_i(x))\paren*{\max\curl*{0, 1 + h(x,\hh_i(x))}^2-\max\curl*{0, 1 + h(x, \pp_i(x))+\mu}^2}\bigg\}\\
& \geq \frac14 \paren*{q(x,\pp_i(x))-q(x, \hh_i(x))}^2
\tag{differentiating with respect to $\mu$, $h$ to optimize}\\
& = \frac14 \paren*{p(x,\pp_i(x))-p(x, \hh_i(x))}^2
\end{align*}
Therefore, by Lemma~\ref{lemma:regret-target}, the conditional regret of the top-$k$ loss can be upper bounded as follows:
\begin{equation*}
\Delta \sC_{\ell_k, \sH}(h, x) = \sum_{i = 1}^k \paren*{p(x, \pp_i(x)) - p(x, \hh_i(x))} \leq 2 k \paren*{\Delta\sC_{\ell_{\rm{sq-hinge}}, \sH}(h, x) }^{\frac12}.
\end{equation*}
By the concavity, take expectations on both sides of the preceding equation, we obtain
\begin{equation*}
\sE_{\ell_k}(h) - \sE^*_{\ell_k}(\sH) + \sM_{\ell_k}(\sH) \leq 2 k\paren*{ \sE_{\ell_{\rm{sq-hinge}}}(h) - \sE^*_{\ell_{\rm{sq-hinge}}}(\sH) + \sM_{\ell_{\rm{sq-hinge}}}(\sH) }^{\frac12}.
\end{equation*}
The second
part follows from the fact that when the hypothesis set $\sH$ is
sufficiently rich such that $\sA_{\ell_{\rm{sq-hinge}}}(\sH) = 0$, we
have $\sM_{\ell_{\rm{sq-hinge}}}(\sH) = 0$.
\end{proof}

\subsection{Proof of Theorem~\ref{thm:bound-hinge}}
\label{app:bound-hinge}

% \subsection{Constrained hinge loss}

Similarly, we study the constrained hinge loss, defined as
$\ell_{\rm{hinge}}(h, x, y) = \sum_{y'\neq y} \max \curl*{0, 1 +
  h(x, y')}$.  The following result shows that $\ell_{\rm{hinge}}$
admits an $\sH$-consistency bound with respect to $\ell_k$. The second
part follows from the fact that when the hypothesis set $\sH$ is
sufficiently rich such that $\sA_{\ell_{\rm{hinge}}}(\sH) = 0$, we
have $\sM_{\ell_{\rm{hinge}}}(\sH) = 0$.  Different from the
constrained squared hinge loss, the bound for $\ell_{\rm{hinge}}$ is
linear: $\sE_{\ell_{\rm{hinge}}}(h) - \sE^*_{\ell_{\rm{hinge}}}(\sH)
\leq \e \Rightarrow \sE_{\ell_k}(h) - \sE^*_{\ell_k}(\sH) \leq k \,
\e$. This also implies that $\ell_{\rm{hinge}}$ is Bayes-consistent
with respect to $\ell_k$.
\begin{restatable}{theorem}{BoundHinge}
\label{thm:bound-hinge}
Assume that $\sH$ is symmetric and complete. Then, for any $1 \leq k
\leq n$, the following $\sH$-consistency bound holds for the
constrained hinge loss:
\begin{align*}
\sE_{\ell_k}(h) - \sE^*_{\ell_k}(\sH) + \sM_{\ell_k}(\sH) \leq k \paren*{\sE_{\ell_{\rm{hinge}}}(h) - \sE^*_{\ell_{\rm{hinge}}}(\sH) + \sM_{\ell_{\rm{hinge}}}(\sH)}.
\end{align*}
In the special case where $\sA_{\ell_{\rm{hinge}}}(\sH) = 0$, for any $1 \leq k \leq n$, the following bound holds:
\begin{align*}
\sE_{\ell_k}(h) - \sE^*_{\ell_k}(\sH) \leq k \paren*{ \sE_{\ell_{\rm{hinge}}}(h) - \sE^*_{\ell_{\rm{hinge}}}(\sH) }.
\end{align*}
\end{restatable}
\begin{proof}
For the constrained hinge loss $\ell_{\rm{hinge}}$, the conditional regret can be written as 
\begin{align*}
\Delta\sC_{\ell_{\rm{hinge}}, \sH}(h, x) 
& = \sum_{y = 1}^n p(x, y) \ell_{\rm{hinge}}(h, x, y) - \inf_{h \in \sH} \sum_{y = 1}^n p(x, y) \ell_{\rm{hinge}}(h, x, y)\\
& \geq \sum_{y = 1}^n p(x, y) \ell_{\rm{hinge}}(h, x, y) - \inf_{\mu \in \Rset} \sum_{y = 1}^n p(x, y) \ell_{\rm{hinge}}(h_{\mu, i}, x, y),
\end{align*}
where for any $i \in [k]$, $h_{\mu, i}(x, y) = \begin{cases}
h(x, y), & y \notin \curl*{\pp_i(x), \hh_i(x)}\\
h(x, \pp_i(x)) + \mu & y = \hh_i(x)\\
h(x, \hh_i(x)) - \mu & y = \pp_i(x).
\end{cases}$
Note that such a choice of $h_{\mu, i}$ leads to the following equality holds:
\begin{equation*}
\sum_{y \notin \curl*{\hh_i(x), \pp_i(x)}} p(x, y) \ell_{\rm{hinge}}(h, x, y) = \sum_{y \notin \curl*{\hh_i(x), \pp_i(x)}}  p(x, y) \ell_{\rm{hinge}}(h_{\mu, i}, x, y).
\end{equation*}
Let $q(x, \pp_i(x)) = 1 - p(x, \pp_i(x))$ and $q(x, \hh_i(x)) = 1 - p(x, \hh_i(x))$.
Therefore, for any $i \in [k]$, the conditional regret of the constrained hinge loss can be lower bounded as
\begin{align*}
\Delta\sC_{\ell_{\rm{hinge}}, \sH}(h, x) 
&   \geq  \inf_{h \in \sH}  \sup_{\mu\in \Rset} \bigg\{q(x, \pp_i(x))\paren*{\max\curl*{0, 1 + h(x, \pp_i(x))}-\max\curl*{0, 1 + h(x,\hh_i(x))-\mu} }\\
& \qquad + q(x,\hh_i(x))\paren*{\max\curl*{0, 1 + h(x,\hh_i(x))}-\max\curl*{0, 1 + h(x, \pp_i(x))+\mu}}\bigg\}\\
& \geq  q(x,\hh_i(x))-q(x, \pp_i(x))
\tag{differentiating with respect to $\mu$, $h$ to optimize}\\
& = p(x,\pp_i(x))-p(x, \hh_i(x))
\end{align*}
Therefore, by Lemma~\ref{lemma:regret-target}, the conditional regret of the top-$k$ loss can be upper bounded as follows:
\begin{equation*}
\Delta \sC_{\ell_k, \sH}(h, x) = \sum_{i = 1}^k \paren*{p(x, \pp_i(x)) - p(x, \hh_i(x))} \leq k \Delta\sC_{\ell_{\rm{hinge}}, \sH}(h, x).
\end{equation*}
By the concavity, take expectations on both sides of the preceding equation, we obtain
\begin{equation*}
\sE_{\ell_k}(h) - \sE^*_{\ell_k}(\sH) + \sM_{\ell_k}(\sH) \leq k \paren*{ \sE_{\ell_{\rm{hinge}}}(h) - \sE^*_{\ell_{\rm{hinge}}}(\sH) + \sM_{\ell_{\rm{hinge}}}(\sH)}.
\end{equation*}
The second
part follows from the fact that when the hypothesis set $\sH$ is
sufficiently rich such that $\sA_{\ell_{\rm{hinge}}}(\sH) = 0$, we
have $\sM_{\ell_{\rm{hinge}}}(\sH) = 0$. 
\end{proof}

\subsection{Proof of Theorem~\ref{thm:bound-rho}}
\label{app:bound-rho}

% \subsection{Constrained $\rho$-margin loss}

The constrained $\rho$-margin loss is defined as $\ell_{\rho}(h, x, y)
= \sum_{y'\neq y} \min\curl*{\max \curl*{0, 1 + h(x, y') / \rho},
  1}$. Next, we show that that $\ell_{\rho}$ benefits form
$\sH$-consistency bounds as well.  The second part
follows from the fact that when the hypothesis set $\sH$ is
sufficiently rich such that $\sA_{\ell_{\rho}}(\sH) = 0$, we have
$\sM_{\ell_{\rho}}(\sH) = 0$. As with the constrained hinge loss, the
bound for $\ell_{\rho}$ is linear: $\sE_{\ell_{\rho}}(h) -
\sE^*_{\ell_{\rho}}(\sH) \leq \e \Rightarrow \sE_{\ell_k}(h) -
\sE^*_{\ell_k}(\sH) \leq k \, \e$. As a by-product, $\ell_{\rho}$ is
Bayes-consistent with respect to $\ell_k$.
\begin{restatable}{theorem}{BoundRho}
\label{thm:bound-rho}
Assume that $\sH$ is symmetric and complete. Then, for any $1 \leq k
\leq n$, the following $\sH$-consistency bound holds for the
constrained $\rho$-margin loss:
\begin{align*}
\sE_{\ell_k}(h) - \sE^*_{\ell_k}(\sH) + \sM_{\ell_k}(\sH) \leq k\, \paren*{\sE_{\ell_{\rho}}(h) - \sE^*_{\ell_{\rho}}(\sH) + \sM_{\ell_{\rho}}(\sH)}.
\end{align*}
In the special case where $\sA_{\ell_{\rho}}(\sH) = 0$, for any $1 \leq k \leq n$, the following bound holds:
\begin{align*}
\sE_{\ell_k}(h) - \sE^*_{\ell_k}(\sH) \leq k\, \paren*{ \sE_{\ell_{\rho}}(h) - \sE^*_{\ell_{\rho}}(\sH) }.
\end{align*}
\end{restatable}

\begin{proof}
For the constrained $\rho$-margin loss $\ell_{\rho}$, the conditional regret can be written as 
\begin{align*}
\Delta\sC_{\ell_{\rho}, \sH}(h, x) 
& = \sum_{y = 1}^n p(x, y) \ell_{\rho}(h, x, y) - \inf_{h \in \sH} \sum_{y = 1}^n p(x, y) \ell_{\rho}(h, x, y)\\
& \geq \sum_{y = 1}^n p(x, y) \ell_{\rho}(h, x, y) - \inf_{\mu \in \Rset} \sum_{y = 1}^n p(x, y) \ell_{\rho}(h_{\mu, i}, x, y),
\end{align*}
where for any $i \in [k]$, $h_{\mu, i}(x, y) = \begin{cases}
h(x, y), & y \notin \curl*{\pp_i(x), \hh_i(x)}\\
h(x, \pp_i(x)) + \mu & y = \hh_i(x)\\
h(x, \hh_i(x)) - \mu & y = \pp_i(x).
\end{cases}$
Note that such a choice of $h_{\mu, i}$ leads to the following equality holds:
\begin{equation*}
\sum_{y \notin \curl*{\hh_i(x), \pp_i(x)}} p(x, y) \ell_{\rho}(h, x, y) = \sum_{y \notin \curl*{\hh_i(x), \pp_i(x)}}  p(x, y) \ell_{\rho}(h_{\mu, i}, x, y).
\end{equation*}
Let $q(x, \pp_i(x)) = 1 - p(x, \pp_i(x))$ and $q(x, \hh_i(x)) = 1 - p(x, \hh_i(x))$.
Therefore, for any $i \in [k]$, the conditional regret of the constrained $\rho$-margin loss can be lower bounded as
\begin{align*}
& \Delta\sC_{\ell_{\rho}, \sH}(h, x)\\ 
&  \geq  \inf_{h \in \sH} \sup_{\mu\in \Rset} \bigg\{q(x, \pp_i(x))\paren*{\min\curl*{\max\curl*{0,1 + \frac{h(x, \pp_i(x))}{\rho}},1}-\min\curl*{\max\curl*{0,1 + \frac{h(x,\hh_i(x))-\mu}{\rho}},1}}\\
&+q(x,\hh_i(x))\paren*{\min\curl*{\max\curl*{0,1 + \frac{h(x,\hh_i(x))}{\rho}},1}-\min\curl*{\max\curl*{0,1 + \frac{h(x, \pp_i(x))+\mu}{\rho}},1}}\bigg\}\\
& \geq q(x,\hh_i(x))-q(x, \pp_i(x))
\tag{differentiating with respect to $\mu$, $h$ to optimize}\\
& = p(x,\pp_i(x))-p(x, \hh_i(x))
\end{align*}
Therefore, by Lemma~\ref{lemma:regret-target}, the conditional regret of the top-$k$ loss can be upper bounded as follows:
\begin{equation*}
\Delta \sC_{\ell_k, \sH}(h, x) = \sum_{i = 1}^k \paren*{p(x, \pp_i(x)) - p(x, \hh_i(x))} \leq k \Delta\sC_{\ell_{\rho}, \sH}(h, x).
\end{equation*}
By the concavity, take expectations on both sides of the preceding equation, we obtain
\begin{equation*}
\sE_{\ell_k}(h) - \sE^*_{\ell_k}(\sH) + \sM_{\ell_k}(\sH) \leq k \paren*{ \sE_{\ell_{\rho}}(h) - \sE^*_{\ell_{\rho}}(\sH) + \sM_{\ell_{\rho}}(\sH)}.
\end{equation*}
The second part
follows from the fact that when the hypothesis set $\sH$ is
sufficiently rich such that $\sA_{\ell_{\rho}}(\sH) = 0$, we have
$\sM_{\ell_{\rho}}(\sH) = 0$.
\end{proof}

\section{Proofs of \texorpdfstring{$\sR$}{R}-consistency bounds for cost-sensitive losses}
\label{app:cost}

We first characterize the best-in class conditional error and the
conditional regret of the target cardinality aware loss function \eqref{eq:target-cardinality}, which will be used in the analysis
of $\sR$-consistency bounds.
\begin{restatable}{lemma}{RegretTargetCost}
\label{lemma:regret-target-cost}
Assume that $\sR$ is symmetric and complete. Then, for any $r \in \sK$ and $x \in
\sX$, the best-in class conditional error and the conditional regret
of the target cardinality aware loss function can be expressed as follows:
\begin{align*}
  \sC^*_{\wt \ell}(\sR, x)
  & = \min_{k  \in \sK} \sum_{y\in \sY}  p(x, y) c(x, k, y)\\
\Delta \sC_{\ell_k, \sH}(r, x)
& = \sum_{y\in \sY}  p(x, y) c(x, \rr(x), y) - \min_{k  \in \sK} \sum_{y\in \sY}  p(x, y) c(x, k, y).
\end{align*}
\end{restatable}
\begin{proof}
By definition, for any $r \in \sR$ and $x \in \sX$, the conditional
error of the target cardinality aware loss function can be written as
\begin{equation*}
\sC_{ \wt \ell }(r, x) =  \sum_{y\in \sY} p(x,y) c(x, \rr(x), y).
\end{equation*}
Since $\sR$ is symmetric and complete, we have
\begin{equation*}
  \sC^*_{\wt \ell }(r, x)
  = \inf_{r \in \sR} \sum_{y\in \sY} p(x,y) c(x, \rr(x), y) = \min_{k  \in \sK} \sum_{i = 1}^k p(x, y) c(x, k, y).
\end{equation*}
Furthermore, the calibration gap can be expressed as
\begin{align*}
\Delta\sC_{\wt \ell, \sH}(r, x)  = \sC_{ \wt \ell }(r, x) - \sC^*_{\wt \ell}(\sR, x) = \sum_{y\in \sY}  p(x, y) c(x, \rr(x), y) - \min_{k  \in \sK} \sum_{y\in \sY}  p(x, y) c(x, k, y),
\end{align*}
which completes the proof.
\end{proof}

\subsection{Proof of Theorem~\ref{thm:bound-cost-comp}}
\label{app:bound-cost-comp}
For convenience, we let $\ov c(x, k, y) = 1 - c(x, k, y)$, $\ov q(x, k) = \sum_{y\in \sY} p(x, y) \ov c(x, k, y) \in [0, 1]$ and $\sS(x, k) = \frac{e^{r(x, k)}}{\sum_{k' \in \sK}e^{r(x, k')}}$. We also let $k_{\min}(x)  =  \argmin_{k \in \sK} \paren*{1 - \ov q(x, k)} = \argmin_{k \in \sK} \sum_{y\in \sY} p(x, y) c(x, k, y)$. 
\BoundCostComp*
\begin{proof}
\textbf{Case I: $\ell = \wt \ell_{\rm{log}}$.} For the cost-sensitive logistic loss $\wt \ell_{\rm{log}}$, the conditional error can be written as 
\begin{align*}
\sC_{\wt \ell_{\rm{log}} }(r, x) = -\sum_{y\in \sY} p(x, y) \sum_{k \in \sK} \ov c(x, k, y) \log\paren*{\frac{e^{r(x, k)}}{\sum_{k'\in \sK}e^{r(x, k')}}}  =   - \sum_{k \in \sK}\log\paren*{\sS(x, k)}\ov q(x, k).
\end{align*}
The conditional regret can be written as 
\begin{align*}
\Delta\sC_{\wt \ell_{\rm{log}}, \sR}(r, x) 
& = - \sum_{k \in \sK}\log\paren*{\sS(x, k)}\ov q(x, k) - \inf_{r \in \sR} \paren*{- \sum_{k \in \sK}\log\paren*{\sS(x, k)}\ov q(x, k)}\\
& \geq - \sum_{k \in \sK}\log\paren*{\sS(x, k)}\ov q(x, k) - \inf_{\mu \in \bracket*{-\sS(x, k_{\min}(x)), \sS(x, \rr(x))}} \paren*{- \sum_{k \in \sK}\log\paren*{\sS_{\mu}(x, k)}\ov q(x, k)},
\end{align*}
where for any $x \in \sX$ and $k \in \sK$, $\sS_{\mu}(x, k) = \begin{cases}
\sS(x, y), & y \notin \curl*{k_{\min}(x), \rr(x)}\\
\sS(x, k_{\min}(x)) + \mu & y = \rr(x)\\
\sS(x, \rr(x)) - \mu & y = k_{\min}(x).
\end{cases}$
Note that such a choice of $\sS_{\mu}$ leads to the following equality holds:
\begin{equation*}
\sum_{k \notin \curl*{\rr(x), k_{\min}(x)}} \log\paren*{\sS(x, k)}\ov q(x, k) = \sum_{k \notin \curl*{\rr(x), k_{\min}(x)}}  \log\paren*{\sS_{\mu}(x, k)}\ov q(x, k).
\end{equation*}
Therefore, the conditional regret of cost-sensitive logistic loss can be lower bounded as
\begin{align*}
\Delta\sC_{\wt \ell_{\rm{log}}, \sH}(h, x)  & \geq \sup_{\mu \in [-\sS(x, k_{\min}(x)),\sS(x,\rr(x))]} \bigg\{\ov q(x, k_{\min}(x))\bracket*{-\log\paren*{\sS(x, k_{\min}(x))} + \log\paren*{\sS(x,\rr(x)) - \mu}}\\
& \qquad + \ov q(x,\rr(x))\bracket*{-\log\paren*{\sS(x,\rr(x))}
    + \log\paren*{\sS(x, k_{\min}(x))+\mu}}\bigg\}.
\end{align*}
By the concavity of the function, differentiate with respect to $\mu$, we obtain that the supremum is achieved by $\mu^* = \frac{\ov q(x,\rr(x))\sS(x,\rr(x))-\ov q(x, k_{\min}(x))\sS(x,
  k_{\min}(x))}{\ov q(x, k_{\min}(x))+\ov q(x,\rr(x))}$. Plug in $\mu^*$, we obtain
\begin{align*}
& \Delta\sC_{\wt \ell_{\rm{log}}, \sH}(h, x)\\ 
& \geq \ov q(x, k_{\min}(x))\log\frac{\paren*{\sS(x,\rr(x))+\sS(x, k_{\min}(x))}\ov q(x, k_{\min}(x))}{\sS(x, k_{\min}(x))\paren*{\ov q(x, k_{\min}(x))+\ov q(x,\rr(x))}}\\
& \qquad +\ov q(x,\rr(x))\log\frac{\paren*{\sS(x,\rr(x))+\sS(x, k_{\min}(x))}\ov q(x,\rr(x))}{\sS(x,\rr(x))\paren*{\ov q(x, k_{\min}(x))+\ov q(x,\rr(x))}}\\
& \geq \ov q(x, k_{\min}(x))\log\frac{2\ov q(x, k_{\min}(x))}{\ov q(x, k_{\min}(x))+\ov q(x,\rr(x))} +\ov q(x,\rr(x))\log\frac{2\ov q(x,\rr(x))}{\ov q(x, k_{\min}(x))+\ov q(x,\rr(x))}
\tag{minimum is achieved when $\sS(x, \rr(x)) = \sS(x, k_{\min}(x))$}\\
& \geq \frac{\paren*{\ov q(x,\rr(x))-\ov q(x, k_{\min}(x))}^2}{2\paren*{\ov q(x,\rr(x))+\ov q(x, k_{\min}(x))}}
\tag{$a\log \frac{2a}{a+b}+b\log \frac{2b}{a+b}\geq \frac{(a-b)^2}{2(a+b)}, \forall a,b\in[0,1]$ \citep[Proposition~E.7]{mohri2018foundations}}\\
& \geq \frac{\paren*{\ov q(x,\rr(x))-\ov q(x, k_{\min}(x))}^2}{4} \tag{$0\leq \ov q(x,\rr(x))+\ov q(x, k_{\min}(x))\leq 2$}.
\end{align*}
Therefore, by Lemma~\ref{lemma:regret-target-cost}, the conditional regret of the target cardinality aware loss function can be upper bounded as follows:
\begin{equation*}
\Delta \sC_{\wt \ell, \sH}(r, x) =  \ov q(x, k_{\min}(x)) - \ov q(x, \rr(x)) \leq 2 \paren*{\Delta\sC_{\wt \ell_{\rm{log}}, \sR}(r, x) }^{\frac12}.
\end{equation*}
By the concavity, take expectations on both sides of the preceding equation, we obtain
\begin{equation*}
\sE_{\wt \ell}(r) - \sE^*_{\wt \ell}(\sR) + \sM_{\wt \ell}(\sR) \leq 2 \paren*{ \sE_{\wt \ell_{\rm{log}}}(r) - \sE^*_{\wt \ell_{\rm{log}}}(\sR) + \sM_{\wt \ell_{\rm{log}}}(\sR) }^{\frac12}.
\end{equation*}
The second part follows from the fact that 
$\sM_{\wt \ell_{\log}}(\sR_{\rm{all}}) = 0$.

\textbf{Case II: $\ell = \wt \ell^{\rm{comp}}_{\rm{exp}}$.} For the cost-sensitive sum exponential loss $\wt \ell^{\rm{comp}}_{\exp}$, the conditional error can be written as 
\begin{align*}
\sC_{\wt \ell^{\rm{comp}}_{\rm{exp}} }(r, x) = \sum_{y\in \sY} p(x, y) \sum_{k \in \sK} \ov c(x, k, y) \sum_{k'\neq k'}e^{r(x, k')-r(x, k)}  =   \sum_{k \in \sK}\paren*{\frac{1}{\sS(x, k)}-1}\ov q(x, k).
\end{align*}
The conditional regret can be written as 
\begin{align*}
\Delta \sC_{\wt \ell^{\rm{comp}}_{\rm{exp}}, \sR}(r, x) 
& = \sum_{k \in \sK}\paren*{\frac{1}{\sS(x, k)}-1}\ov q(x, k) - \inf_{r \in \sR} \paren*{\sum_{k \in \sK}\paren*{\frac{1}{\sS(x, k)}-1}\ov q(x, k)}\\
& \geq \sum_{k \in \sK}\paren*{\frac{1}{\sS(x, k)}-1}\ov q(x, k) - \inf_{\mu \in \bracket*{-\sS(x, k_{\min}(x)), \sS(x, \rr(x))}} \paren*{\sum_{k \in \sK}\paren*{\frac{1}{\sS_{\mu}(x, k)}-1}\ov q(x, k)},
\end{align*}
where for any $x \in \sX$ and $k \in \sK$, $\sS_{\mu}(x, k) = \begin{cases}
\sS(x, y), & y \notin \curl*{k_{\min}(x), \rr(x)}\\
\sS(x, k_{\min}(x)) + \mu & y = \rr(x)\\
\sS(x, \rr(x)) - \mu & y = k_{\min}(x).
\end{cases}$
Note that such a choice of $\sS_{\mu}$ leads to the following equality holds:
\begin{equation*}
\sum_{k \notin \curl*{\rr(x), k_{\min}(x)}} \paren*{\frac{1}{\sS(x, k)}-1}\ov q(x, k) = \sum_{k \notin \curl*{\rr(x), k_{\min}(x)}}  \paren*{\frac{1}{\sS_{\mu}(x, k)}-1}\ov q(x, k).
\end{equation*}
Therefore, the conditional regret of cost-sensitive sum exponential loss can be lower bounded as
\begin{align*}
\Delta\sC_{\wt \ell^{\rm{comp}}_{\rm{exp}}, \sH}(h, x)  & \geq \sup_{\mu \in [-\sS(x, k_{\min}(x)),\sS(x,\rr(x))]} \bigg\{\ov q(x, k_{\min}(x))\bracket*{\frac{1}{\sS(x, k_{\min}(x))}-\frac{1}{\sS(x,\rr(x))-\mu}}\\
&\qquad +\ov q(x, \rr(x))\bracket*{\frac{1}{\sS(x,\rr(x))}-\frac{1}{\sS(x, k_{\min}(x))+ \mu}}\bigg\}.
\end{align*}
By the concavity of the function, differentiate with respect to $\mu$, we obtain that the supremum is achieved by $\mu^* = \frac{\sqrt{\ov q(x,\rr(x)})\sS(x,\rr(x))-\sqrt{\ov q(x, k_{\min}(x))}\sS(x, k_{\min}(x))}{\sqrt{\ov q(x, k_{\min}(x))}+ \sqrt{\ov q(x, \rr(x))}}$. Plug in $\mu^*$, we obtain
\begin{align*}
& \Delta\sC_{\wt \ell^{\rm{comp}}_{\rm{exp}}, \sH}(h, x)\\ 
& \geq \frac{\ov q(x, k_{\min}(x))}{\sS(x, k_{\min}(x))}+ \frac{\ov q(x, \rr(x)))}{\sS(x,\rr(x)))}-\frac{\paren*{\sqrt{\ov q(x, k_{\min}(x))}+ \sqrt{\ov q(x, \rr(x)))}}^2}{\sS(x, k_{\min}(x))+ \sS(x,\rr(x)))}\\
& \geq \paren*{\sqrt{\ov q(x, k_{\min}(x))}-\sqrt{\ov q(x, \rr(x)))}}^2
\tag{minimum is achieved when $\sS(x, \rr(x)) = \sS(x, k_{\min}(x)) = \frac12$}\\
& \geq \frac{\paren*{\ov q(x, \rr(x)))-\ov q(x, k_{\min}(x))}^2}{\paren*{\sqrt{\ov q(x, \rr(x)))}+ \sqrt{\ov q(x, k_{\min}(x))}}^2}\\
& \geq \frac{\paren*{\ov q(x, \rr(x)))-\ov q(x, k_{\min}(x))}^2}{4}
\tag{$\sqrt{a}+ \sqrt{b}\leq 2, \forall a,b\in[0,1], a+b\leq 2$}.
\end{align*}
Therefore, by Lemma~\ref{lemma:regret-target-cost}, the conditional regret of the target cardinality aware loss function can be upper bounded as follows:
\begin{equation*}
\Delta \sC_{\wt \ell, \sH}(r, x) =  \ov q(x, k_{\min}(x)) - \ov q(x, \rr(x)) \leq 2 \paren*{\Delta\sC_{\wt \ell^{\rm{comp}}_{\rm{exp}}, \sR}(r, x) }^{\frac12}.
\end{equation*}
By the concavity, take expectations on both sides of the preceding equation, we obtain
\begin{equation*}
\sE_{\wt \ell}(r) - \sE^*_{\wt \ell}(\sR) + \sM_{\wt \ell}(\sR) \leq 2 \paren*{ \sE_{\wt \ell^{\rm{comp}}_{\rm{exp}}}(r) - \sE^*_{\wt \ell^{\rm{comp}}_{\rm{exp}}}(\sR) + \sM_{\wt \ell^{\rm{comp}}_{\rm{exp}}}(\sR) }^{\frac12}.
\end{equation*}
The second part follows from the fact that 
$\sM_{\wt \ell^{\rm{comp}}_{\rm{exp}}}(\sR_{\rm{all}}) = 0$.

\textbf{Case III: $\ell = \wt \ell_{\rm{gce}}$.} For the cost-sensitive generalized cross-entropy  loss $\wt \ell_{\rm{gce}}$, the conditional error can be written as 
\begin{align*}
\sC_{\wt \ell_{\rm{gce}} }(r, x) = \sum_{y\in \sY} p(x, y) \sum_{k \in \sK} \ov c(x, k, y)\frac{1}{\alpha}\paren*{1 - \paren*{\frac{e^{r(x, k)}}{\sum_{k'\in \sK}e^{r(x, k')}}}^{\alpha}}  =   \frac{1}{\alpha} \sum_{k \in \sK}\paren*{1 - \sS(x, k)^{\alpha}}\ov q(x, k).
\end{align*}
The conditional regret can be written as 
\begin{align*}
\Delta\sC_{\wt \ell_{\rm{gce}}, \sR}(r, x) 
& = \frac{1}{\alpha} \sum_{k \in \sK}\paren*{1 - \sS(x, k)^{\alpha}}\ov q(x, k) - \inf_{r \in \sR} \paren*{\frac{1}{\alpha} \sum_{k \in \sK}\paren*{1 - \sS(x, k)^{\alpha}}\ov q(x, k)}\\
& \geq \frac{1}{\alpha} \sum_{k \in \sK}\paren*{1 - \sS(x, k)^{\alpha}}\ov q(x, k) - \inf_{\mu \in \bracket*{-\sS(x, k_{\min}(x)), \sS(x, \rr(x))}} \paren*{\frac{1}{\alpha} \sum_{k \in \sK}\paren*{1 - \sS_{\mu}(x, k)^{\alpha}}\ov q(x, k)},
\end{align*}
where for any $x \in \sX$ and $k \in \sK$, $\sS_{\mu}(x, k) = \begin{cases}
\sS(x, y), & y \notin \curl*{k_{\min}(x), \rr(x)}\\
\sS(x, k_{\min}(x)) + \mu & y = \rr(x)\\
\sS(x, \rr(x)) - \mu & y = k_{\min}(x).
\end{cases}$
Note that such a choice of $\sS_{\mu}$ leads to the following equality holds:
\begin{equation*}
\sum_{k \notin \curl*{\rr(x), k_{\min}(x)}} \frac{1}{\alpha} \sum_{k \in \sK}\paren*{1 - \sS(x, k)^{\alpha}}\ov q(x, k) = \sum_{k \notin \curl*{\rr(x), k_{\min}(x)}}  \frac{1}{\alpha} \sum_{k \in \sK}\paren*{1 - \sS_{\mu}(x, k)^{\alpha}}\ov q(x, k).
\end{equation*}
Therefore, the conditional regret of cost-sensitive generalized cross-entropy loss can be lower bounded as
\begin{align*}
\Delta\sC_{\wt \ell_{\rm{gce}}, \sH}(h, x)  &   =  \frac{1}{\alpha} \sup_{\mu \in [-\sS(x, k_{\min}(x)),\sS(x, \rr(x))]} \bigg\{\ov q(x, k_{\min}(x))\bracket*{-\sS(x, k_{\min}(x))^{\alpha}+\paren*{\sS(x, \rr(x))-\mu}^{\alpha}}\\
&\qquad +\ov q(x, \rr(x))\bracket*{-\sS(x, \rr(x))^{\alpha}+ \paren*{\sS(x, k_{\min}(x))+ \mu}^{\alpha}}\bigg\}.
\end{align*}
By the concavity of the function, differentiate with respect to $\mu$, we obtain that the supremum is achieved by $\mu^* = \frac{\ov q(x, \rr(x))^{\frac{1}{1-\alpha}}\sS(x, \rr(x))-\ov q(x, k_{\min}(x))^{\frac{1}{1-\alpha}}\sS(x, k_{\min}(x))}{\ov q(x, k_{\min}(x))^{\frac{1}{1-\alpha}}+\ov q(x, \rr(x))^{\frac{1}{1-\alpha}}}$. Plug in $\mu^*$, we obtain
\begin{align*}
& \Delta\sC_{\wt \ell_{\rm{gce}}, \sH}(h, x)\\ 
&   \geq  \frac{1}{\alpha}\paren*{\sS(x, \rr(x))+ \sS(x, k_{\min}(x))}^{\alpha}\paren*{\ov q(x, k_{\min}(x))^{\frac{1}{1-\alpha}}+\ov q(x, \rr(x))^{\frac{1}{1-\alpha}}}^{1-\alpha}\\
&\qquad-\frac{1}{\alpha}\ov q(x, k_{\min}(x))\sS(x, k_{\min}(x))^{\alpha}-\frac{1}{\alpha}\ov q(x, \rr(x))\sS(x, \rr(x))^{\alpha}\\
& \geq \frac{1}{\alpha n^{\alpha}}\bracket*{2^{\alpha}\paren*{\ov q(x, k_{\min}(x))^{\frac{1}{1-\alpha}}+\ov q(x, \rr(x))^{\frac{1}{1-\alpha}}}^{1-\alpha}-\ov q(x, k_{\min}(x))-\ov q(x, \rr(x))}
\tag{minimum is achieved when $\sS(x, \rr(x)) = \sS(x, k_{\min}(x)) = \frac1n$}\\
& \geq \frac{\paren*{\ov q(x, \rr(x))-\ov q(x, k_{\min}(x))}^2}{4n^{\alpha}}
\tag{$\paren*{\frac{a^{\frac{1}{1-\alpha}}+b^{\frac{1}{1-\alpha}}}{2}}^{1-\alpha}-\frac{a+b}{2}\geq \frac{\alpha}{4}(a-b)^2, \forall a,b\in[0,1]$, $0\leq a+b\leq 1$}.
\end{align*}
Therefore, by Lemma~\ref{lemma:regret-target-cost}, the conditional regret of the target cardinality aware loss function can be upper bounded as follows:
\begin{equation*}
\Delta \sC_{\wt \ell, \sH}(r, x) =  \ov q(x, k_{\min}(x)) - \ov q(x, \rr(x)) \leq 2 n^{\frac{\alpha}{2}} \paren*{\Delta\sC_{\wt \ell_{\rm{gce}}, \sR}(r, x) }^{\frac12}.
\end{equation*}
By the concavity, take expectations on both sides of the preceding equation, we obtain
\begin{equation*}
\sE_{\wt \ell}(r) - \sE^*_{\wt \ell}(\sR) + \sM_{\wt \ell}(\sR) \leq 2 n^{\frac{\alpha}{2}}\paren*{ \sE_{\wt \ell_{\rm{gce}}}(r) - \sE^*_{\wt \ell_{\rm{gce}}}(\sR) + \sM_{\wt \ell_{\rm{gce}}}(\sR) }^{\frac12}.
\end{equation*}
The second part follows from the fact that 
$\sM_{\wt \ell_{\rm{gce}}}(\sR_{\rm{all}}) = 0$.

\textbf{Case IV: $\ell = \wt \ell_{\rm{mae}}$.} For the cost-sensitive mean absolute error loss $\wt \ell_{\rm{mae}}$, the conditional error can be written as 
\begin{align*}
\sC_{\wt \ell_{\rm{mae}} }(r, x) = \sum_{y\in \sY} p(x, y) \sum_{k \in \sK} \ov c(x, k, y) \paren*{1 - \paren*{\frac{e^{r(x, k)}}{\sum_{k'\in \sK}e^{r(x, k')}}}}  =  \sum_{k \in \sK}\paren*{1 - \sS(x, k)}\ov q(x, k).
\end{align*}
The conditional regret can be written as 
\begin{align*}
\Delta\sC_{\wt \ell_{\rm{mae}}, \sR}(r, x) 
& =\sum_{k \in \sK}\paren*{1 - \sS(x, k)}\ov q(x, k) - \inf_{r \in \sR} \paren*{\sum_{k \in \sK}\paren*{1 - \sS(x, k)}\ov q(x, k)}\\
& \geq \sum_{k \in \sK}\paren*{1 - \sS(x, k)}\ov q(x, k) - \inf_{\mu \in \bracket*{-\sS(x, k_{\min}(x)), \sS(x, \rr(x))}} \paren*{\sum_{k \in \sK}\paren*{1 - \sS_{\mu}(x, k)}\ov q(x, k)},
\end{align*}
where for any $x \in \sX$ and $k \in \sK$, $\sS_{\mu}(x, k) = \begin{cases}
\sS(x, y), & y \notin \curl*{k_{\min}(x), \rr(x)}\\
\sS(x, k_{\min}(x)) + \mu & y = \rr(x)\\
\sS(x, \rr(x)) - \mu & y = k_{\min}(x).
\end{cases}$
Note that such a choice of $\sS_{\mu}$ leads to the following equality holds:
\begin{equation*}
\sum_{k \in \sK}\paren*{1 - \sS(x, k)}\ov q(x, k) = \sum_{k \in \sK}\paren*{1 - \sS_{\mu}(x, k)}\ov q(x, k).
\end{equation*}
Therefore, the conditional regret of cost-sensitive mean absolute error can be lower bounded as
\begin{align*}
\Delta\sC_{\wt \ell_{\rm{mae}}, \sH}(h, x)  & \geq \sup_{\mu \in [-\sS(x, k_{\min}(x)),\sS(x, \rr(x))]} \bigg\{\ov q(x, k_{\min}(x))\bracket*{-\sS(x, k_{\min}(x))+\sS(x, \rr(x))-\mu}\\
&\qquad +\ov q(x, \rr(x))\bracket*{-\sS(x, \rr(x))+ \sS(x, k_{\min}(x))+ \mu}\bigg\}.
\end{align*}
By the concavity of the function, differentiate with respect to $\mu$, we obtain that the supremum is achieved by $\mu^* = -\sS(x, k_{\min}(x))$. Plug in $\mu^*$, we obtain
\begin{align*}
& \Delta\sC_{\wt \ell_{\rm{mae}}, \sH}(h, x)\\ 
& \geq \ov q(x, k_{\min}(x))\sS(x, \rr(x))-\ov q(x, \rr(x))\sS(x, \rr(x))\\
& \geq \frac{1}{n}\paren*{\ov q(x, k_{\min}(x))-\ov q(x, \rr(x))}
\tag{minimum is achieved when $\sS(x, \rr(x)) = \frac1n$}.
\end{align*}
Therefore, by Lemma~\ref{lemma:regret-target-cost}, the conditional regret of the target cardinality aware loss function can be upper bounded as follows:
\begin{equation*}
\Delta \sC_{\wt \ell, \sH}(r, x) =  \ov q(x, k_{\min}(x)) - \ov q(x, \rr(x)) \leq n \paren*{\Delta\sC_{\wt \ell_{\rm{mae}}, \sR}(r, x) }.
\end{equation*}
By the concavity, take expectations on both sides of the preceding equation, we obtain
\begin{equation*}
\sE_{\wt \ell}(r) - \sE^*_{\wt \ell}(\sR) + \sM_{\wt \ell}(\sR) \leq n \paren*{ \sE_{\wt \ell_{\rm{mae}}}(r) - \sE^*_{\wt \ell_{\rm{mae}}}(\sR) + \sM_{\wt \ell_{\rm{mae}}}(\sR) }.
\end{equation*}
The second part follows from the fact that 
$\sM_{\wt \ell_{\rm{mae}}}(\sR_{\rm{all}}) = 0$.
\end{proof}

\subsection{Proof of Theorem~\ref{thm:bound-cost-cstnd}}
\label{app:bound-cost-cstnd}

The conditional error for the cost-sensitive constrained loss can be expressed as follows:
\begin{align*}
\sC_{\wt \ell^{\rm{cstnd}}}(r, x) 
&=  \sum_{y\in \sY} p(x, y)  \wt \ell^{\rm{cstnd}}(r, x, y)\\
&=  \sum_{y\in \sY} p(x, y) \sum_{k \in \sK} c(x, k, y) \Phi\paren*{-r(x, k)}\\
& = \sum_{k \in \sK} \wt q(x, k)\Phi\paren*{-r(x, k)},
\end{align*}
where $\wt q(x, k) =   \sum_{y\in \sY} p(x, y)  c(x, k, y) \in [0, 1]$. Let $k_{\min}(x)  =  \argmin_{k \in \sK} \wt q(x, k)$.
We denote by $\Phi_{\rm{exp}} \colon t \mapsto e^{-t}$ the exponential loss function, $\Phi_{\rm{sq-hinge}} \colon t \mapsto \max \curl*{0, 1 - t}^2$ the squared hinge loss function, $\Phi_{\rm{hinge}} \colon t
\mapsto \max \curl*{0, 1 - t}$ the hinge loss function, and $\Phi_{\rho} \colon t \mapsto \min\curl*{\max\curl*{0,
    1 - t/\rho}, 1}$, $\rho > 0$ the $\rho$-margin loss function.
\BoundCostCstnd*
\begin{proof}
\textbf{Case I: $\ell = \wt \ell^{\rm{cstnd}}_{\rm{exp}}$.}
For the cost-sensitive constrained exponential loss $\wt \ell^{\rm{cstnd}}_{\rm{exp}}$, the conditional regret can be written as 
\begin{align*}
\Delta\sC_{\wt \ell^{\rm{cstnd}}_{\rm{exp}}, \sR}(r, x) 
& = \sum_{k \in \sK} \wt q(x, k\Phi_{\rm{exp}}\paren*{-r(x, k)} - \inf_{r \in \sR} \sum_{k \in \sK} \wt q(x, k)\Phi_{\rm{exp}}\paren*{-r(x, k)}\\
& \geq \sum_{k \in \sK} \wt q(x, k)\Phi_{\rm{exp}}\paren*{-r(x, k)} - \inf_{\mu \in \Rset} \sum_{k \in \sK} \wt q(x, k)\Phi_{\rm{exp}}\paren*{-r_{\mu}(x, k)},
\end{align*}
where for any $k \in \sK$, $r_{\mu}(x, k) = \begin{cases}
r(x, y), & y \notin \curl*{k_{\min}(x), \rr(x)}\\
r(x, k_{\min}(x)) + \mu & y = \rr(x)\\
r(x, \rr(x)) - \mu & y = k_{\min}(x).
\end{cases}$
Note that such a choice of $r_{\mu}$ leads to the following equality holds:
\begin{equation*}
\sum_{k \notin \curl*{\rr(x), k_{\min}(x)}}  \wt q(x, k)\Phi_{\rm{exp}}\paren*{-r(x, k)} = \sum_{k \notin \curl*{\rr(x), k_{\min}(x)}}  \sum_{k \in \sK} \wt q(x, k)\Phi_{\rm{exp}}\paren*{-r_{\mu}(x, k)}.
\end{equation*}
Therefore, the conditional regret of cost-sensitive constrained exponential loss can be lower bounded as
\begin{align*}
& \Delta\sC_{\wt \ell^{\rm{cstnd}}_{\rm{exp}}, \sR}(r, x)\\ 
& \geq \inf_{r \in \sR} \sup_{\mu\in \Rset} \curl*{\wt q(x, k_{\min}(x))\paren*{e^{r(x, k_{\min}(x))}-e^{r(x,\rr(x))-\mu}}+\wt q(x,\rr(x))\paren*{e^{r(x,\rr(x))}-e^{r(x, k_{\min}(x))+\mu}}}\\
& = \paren*{\sqrt{\wt q(x, k_{\min}(x))}-\sqrt{\wt q(x,\rr(x))}}^2 \tag{differentiating with respect to $\mu$, $r$ to optimize}\\
&   =   \paren*{\frac{\wt q(x,\rr(x)) - \wt q(x, k_{\min}(x))}{\sqrt{\wt q(x, k_{\min}(x))} + \sqrt{\wt q(x, \rr(x))}}}^2\\
& \geq \frac1{4} \paren*{\wt q(x,\rr(x)) - \wt q(x, k_{\min}(x))}^2 \tag{$0 \leq \wt q(x, k)\leq 1$}.
\end{align*}
Therefore, by Lemma~\ref{lemma:regret-target-cost}, the conditional regret of the target cardinality aware loss function can be upper bounded as follows:
\begin{equation*}
\Delta \sC_{\wt \ell, \sH}(r, x) = \wt q(x, \rr(x)) - \wt q(x, k_{\min}(x)) \leq 2 \paren*{\Delta\sC_{\wt \ell^{\rm{cstnd}}_{\rm{exp}}, \sR}(r, x) }^{\frac12}.
\end{equation*}
By the concavity, take expectations on both sides of the preceding equation, we obtain
\begin{equation*}
\sE_{\wt \ell}(r) - \sE^*_{\wt \ell}(\sR) + \sM_{\wt \ell}(\sR) \leq 2 \paren*{ \sE_{\wt \ell^{\rm{cstnd}}_{\rm{exp}}}(r) - \sE^*_{\wt \ell^{\rm{cstnd}}_{\rm{exp}}}(\sR) + \sM_{\wt \ell^{\rm{cstnd}}_{\rm{exp}}}(\sR) }^{\frac12}.
\end{equation*}
The second part follows from the fact that 
$\sM_{\wt \ell^{\rm{cstnd}}_{\exp}}(\sR_{\rm{all}}) = 0$.

\textbf{Case II: $\ell = \wt \ell_{\rm{sq-hinge}}$.}
For the cost-sensitive constrained squared hinge loss $\wt \ell_{\rm{sq-hinge}}$, the conditional regret can be written as 
\begin{align*}
\Delta\sC_{\wt \ell_{\rm{sq-hinge}}, \sR}(r, x) 
& = \sum_{k \in \sK} \wt q(x, k)\Phi_{\rm{sq-hinge}}\paren*{-r(x, k)} - \inf_{r \in \sR} \sum_{k \in \sK} \wt q(x, k)\Phi_{\rm{sq-hinge}}\paren*{-r(x, k)}\\
& \geq \sum_{k \in \sK} \wt q(x, k)\Phi_{\rm{sq-hinge}}\paren*{-r(x, k)} - \inf_{\mu \in \Rset} \sum_{k \in \sK} \wt q(x, k)\Phi_{\rm{sq-hinge}}\paren*{-r_{\mu}(x, k)},
\end{align*}
where for any $k \in \sK$, $r_{\mu}(x, k) = \begin{cases}
r(x, y), & y \notin \curl*{k_{\min}(x), \rr(x)}\\
r(x, k_{\min}(x)) + \mu & y = \rr(x)\\
r(x, \rr(x)) - \mu & y = k_{\min}(x).
\end{cases}$
Note that such a choice of $r_{\mu}$ leads to the following equality holds:
\begin{equation*}
\sum_{k \notin \curl*{\rr(x), k_{\min}(x)}}  \wt q(x, k)\Phi_{\rm{sq-hinge}}\paren*{-r(x, k)} = \sum_{k \notin \curl*{\rr(x), k_{\min}(x)}}  \sum_{k \in \sK} \wt q(x, k)\Phi_{\rm{sq-hinge}}\paren*{-r_{\mu}(x, k)}.
\end{equation*}
Therefore, the conditional regret of cost-sensitive constrained squared hinge loss can be lower bounded as
\begin{align*}
& \Delta\sC_{\wt \ell_{\rm{sq-hinge}}, \sR}(r, x)\\ 
& \geq \inf_{r \in \sR} \sup_{\mu\in \Rset} \bigg\{\wt q(x, k_{\min}(x))\paren*{\max\curl*{0, 1 + r(x, k_{\min}(x))}^2-\max\curl*{0, 1 + r(x,\rr(x))-\mu}^2 }\\
& \qquad + \wt q(x,\rr(x))\paren*{\max\curl*{0, 1 + r(x,\rr(x))}^2-\max\curl*{0, 1 + r(x, k_{\min}(x))+\mu}^2}\bigg\}\\
& \geq \frac14 \paren*{\wt q(x,k_{\min}(x))-\wt q(x, \rr(x))}^2
\tag{differentiating with respect to $\mu$, $r$ to optimize}.
\end{align*}
Therefore, by Lemma~\ref{lemma:regret-target-cost}, the conditional regret of the target cardinality aware loss function can be upper bounded as follows:
\begin{equation*}
\Delta \sC_{\wt \ell, \sH}(r, x) = \wt q(x, \rr(x)) - \wt q(x, k_{\min}(x)) \leq 2 \paren*{\Delta\sC_{\wt \ell_{\rm{sq-hinge}}, \sR}(r, x) }^{\frac12}.
\end{equation*}
By the concavity, take expectations on both sides of the preceding equation, we obtain
\begin{equation*}
\sE_{\wt \ell}(r) - \sE^*_{\wt \ell}(\sR) + \sM_{\wt \ell}(\sR) \leq 2 \paren*{ \sE_{\wt \ell_{\rm{sq-hinge}}}(r) - \sE^*_{\wt \ell_{\rm{sq-hinge}}}(\sR) + \sM_{\wt \ell_{\rm{sq-hinge}}}(\sR) }^{\frac12}.
\end{equation*}
The second part follows from the fact that 
$\sM_{\wt \ell_{\rm{sq-hinge}}}(\sR_{\rm{all}}) = 0$.

\textbf{Case III: $\ell = \wt \ell_{\rm{hinge}}$.}
For the cost-sensitive constrained hinge loss $\wt \ell_{\rm{hinge}}$, the conditional regret can be written as 
\begin{align*}
\Delta\sC_{\wt \ell_{\rm{hinge}}, \sR}(r, x) 
& = \sum_{k \in \sK} \wt q(x, k)\Phi_{\rm{hinge}}\paren*{-r(x, k)} - \inf_{r \in \sR} \sum_{k \in \sK} \wt q(x, k)\Phi_{\rm{hinge}}\paren*{-r(x, k)}\\
& \geq \sum_{k \in \sK} \wt q(x, k)\Phi_{\rm{hinge}}\paren*{-r(x, k)} - \inf_{\mu \in \Rset} \sum_{k \in \sK} \wt q(x, k)\Phi_{\rm{hinge}}\paren*{-r_{\mu}(x, k)},
\end{align*}
where for any $k \in \sK$, $r_{\mu}(x, k) = \begin{cases}
r(x, y), & y \notin \curl*{k_{\min}(x), \rr(x)}\\
r(x, k_{\min}(x)) + \mu & y = \rr(x)\\
r(x, \rr(x)) - \mu & y = k_{\min}(x).
\end{cases}$
Note that such a choice of $r_{\mu}$ leads to the following equality holds:
\begin{equation*}
\sum_{k \notin \curl*{\rr(x), k_{\min}(x)}}  \wt q(x, k)\Phi_{\rm{hinge}}\paren*{-r(x, k)} = \sum_{k \notin \curl*{\rr(x), k_{\min}(x)}}  \sum_{k \in \sK} \wt q(x, k)\Phi_{\rm{hinge}}\paren*{-r_{\mu}(x, k)}.
\end{equation*}
Therefore, the conditional regret of cost-sensitive constrained hinge loss can be lower bounded as
\begin{align*}
\Delta\sC_{\wt \ell_{\rm{hinge}}, \sR}(r, x) 
& \geq \inf_{r \in \sR} \sup_{\mu\in \Rset} \bigg\{q(x, k_{\min}(x))\paren*{\max\curl*{0, 1 + r(x, k_{\min}(x))}-\max\curl*{0, 1 + r(x,\rr(x))-\mu} }\\
& \qquad + q(x,\rr(x))\paren*{\max\curl*{0, 1 + r(x,\rr(x))}-\max\curl*{0, 1 + r(x, k_{\min}(x))+\mu}}\bigg\}\\
& \geq  q(x,\rr(x))-q(x, k_{\min}(x))
\tag{differentiating with respect to $\mu$, $r$ to optimize}.
\end{align*}
Therefore, by Lemma~\ref{lemma:regret-target-cost}, the conditional regret of the target cardinality aware loss function can be upper bounded as follows:
\begin{equation*}
\Delta \sC_{\wt \ell, \sH}(r, x) = \wt q(x, \rr(x)) - \wt q(x, k_{\min}(x)) \leq \Delta\sC_{\wt \ell_{\rm{hinge}}, \sR}(r, x).
\end{equation*}
By the concavity, take expectations on both sides of the preceding equation, we obtain
\begin{equation*}
\sE_{\wt \ell}(r) - \sE^*_{\wt \ell}(\sR) + \sM_{\wt \ell}(\sR) \leq \sE_{\wt \ell_{\rm{hinge}}}(r) - \sE^*_{\wt \ell_{\rm{hinge}}}(\sR) + \sM_{\wt \ell_{\rm{hinge}}}(\sR).
\end{equation*}
The second part follows from the fact that 
$\sM_{\wt \ell_{\rm{hinge}}}(\sR_{\rm{all}}) = 0$.

\textbf{Case IV: $\ell = \wt \ell_{\rho}$.}
For the cost-sensitive constrained $\rho$-margin loss $\wt \ell_{\rho}$, the conditional regret can be written as 
\begin{align*}
\Delta\sC_{\wt \ell_{\rho}, \sR}(r, x) 
& = \sum_{k \in \sK} \wt q(x, k)\Phi_{\rho}\paren*{-r(x, k)} - \inf_{r \in \sR} \sum_{k \in \sK} \wt q(x, k)\Phi_{\rho}\paren*{-r(x, k)}\\
& \geq \sum_{k \in \sK} \wt q(x, k)\Phi_{\rho}\paren*{-r(x, k)} - \inf_{\mu \in \Rset} \sum_{k \in \sK} \wt q(x, k)\Phi_{\rho}\paren*{-r_{\mu}(x, k)},
\end{align*}
where for any $k \in \sK$, $r_{\mu}(x, k) = \begin{cases}
r(x, y), & y \notin \curl*{k_{\min}(x), \rr(x)}\\
r(x, k_{\min}(x)) + \mu & y = \rr(x)\\
r(x, \rr(x)) - \mu & y = k_{\min}(x).
\end{cases}$
Note that such a choice of $r_{\mu}$ leads to the following equality holds:
\begin{equation*}
\sum_{k \notin \curl*{\rr(x), k_{\min}(x)}}  \wt q(x, k)\Phi_{\rho}\paren*{-r(x, k)} = \sum_{k \notin \curl*{\rr(x), k_{\min}(x)}}  \sum_{k \in \sK} \wt q(x, k)\Phi_{\rho}\paren*{-r_{\mu}(x, k)}.
\end{equation*}
Therefore, the conditional regret of cost-sensitive constrained $\rho$-margin loss can be lower bounded as
\begin{align*}
& \Delta\sC_{\wt \ell_{\rho}, \sR}(r, x)\\ 
& \geq \inf_{r \in \sR} \sup_{\mu\in \Rset} \bigg\{\wt q(x, k_{\min}(x))\paren*{\min\curl*{\max\curl*{0,1 + \frac{r(x, k_{\min}(x))}{\rho}},1}-\min\curl*{\max\curl*{0,1 + \frac{r(x,\rr(x))-\mu}{\rho}},1}}\\
&+\wt q(x,\rr(x))\paren*{\min\curl*{\max\curl*{0,1 + \frac{r(x,\rr(x))}{\rho}},1}-\min\curl*{\max\curl*{0,1 + \frac{r(x, k_{\min}(x))+\mu}{\rho}},1}}\bigg\}\\
& \geq \wt q(x,\rr(x))-\wt q(x, k_{\min}(x))
\tag{differentiating with respect to $\mu$, $r$ to optimize}.
\end{align*}
Therefore, by Lemma~\ref{lemma:regret-target-cost}, the conditional regret of the target cardinality aware loss function can be upper bounded as follows:
\begin{equation*}
\Delta \sC_{\wt \ell, \sH}(r, x) = \wt q(x, \rr(x)) - \wt q(x, k_{\min}(x)) \leq \Delta\sC_{\wt \ell_{\rho}, \sR}(r, x).
\end{equation*}
By the concavity, take expectations on both sides of the preceding equation, we obtain
\begin{equation*}
\sE_{\wt \ell}(r) - \sE^*_{\wt \ell}(\sR) + \sM_{\wt \ell}(\sR) \leq \sE_{\wt \ell_{\rho}}(r) - \sE^*_{\wt \ell_{\rho}}(\sR) + \sM_{\wt \ell_{\rho}}(\sR).
\end{equation*}
The second part follows from the fact that 
$\sM_{\wt \ell_{\rho}}(\sR_{\rm{all}}) = 0$.
\end{proof}
\end{document}